\documentclass[preprint]{elsarticle}

\usepackage{amssymb}
\usepackage{amsthm}

\usepackage{amsmath}
\usepackage{amstext}
\usepackage{amsfonts}
\usepackage{xcolor}
\usepackage{graphicx}
\usepackage[shortlabels]{enumitem}

\usepackage[algoruled,lined]{algorithm2e}

\newcommand{\lo}[1]{\ensuremath{\underline{#1}}}
\newcommand{\hi}[1]{\ensuremath{\overline{{#1}}}}
\newcommand{\diff}{\ensuremath{\,\mathrm{d}}}
\newcommand{\kldsym}{{\ensuremath{\mathrm{D}}}}
\newcommand{\kld}[2]{{\ensuremath{\vphantom{f^2}\kldsym\!\left(#1\left.\vphantom{#1#2}\!\right\|#2\right)}}}
\newcommand{\supp}{\ensuremath{\mbox{\textnormal{supp}}}}
\newcommand{\subs}[1]{\ensuremath{_{\mbox{\scriptsize {#1}}}}}       
\newcommand{\sxn}{{\ensuremath{\mbox{\textnormal{S}}}}}       
\newcommand{\sx}{{\ensuremath{\mbox{\tiny{\sxn}}}}}     

\newcommand{\xsize}{\ensuremath{{\ell_{x}}}} 
\newcommand{\usize}{\ensuremath{{\ell_{u}}}} 
\newcommand{\ysize}{\ensuremath{{\ell_{y}}}} 
\newcommand{\zsize}{\ensuremath{{\ell_{z}}}} 

\newcommand{\unitvect}[1]{\ensuremath{\mathrm{1}_{(#1)}}}

\newcommand{\fu}[2][]{f#1\left(#2\right)}              
\newcommand{\f}[3][]{f#1\left(#2\vphantom{#2#3}\right|\!\left. #3\vphantom{#2#3}\right)}

\newcommand{\slo}[1][]{\ensuremath{\underline{s}}}  
\newcommand{\spu}[1][]{\ensuremath{\overline{s}}}   
\newcommand{\mS}[1]{\ensuremath{\mathbb{\uppercase{#1}}}} 
\newcommand{\xnparv}{\ensuremath{\omega}}  
\newcommand{\ynparv}{\ensuremath{\nu}}     
\newcommand{\xnpars}{\ensuremath{\rho}}  
\newcommand{\ynpars}{\ensuremath{r}}     
\newcommand{\anal}{\ensuremath{^{(a)}}}     
\newcommand{\synt}{\ensuremath{^{(s)}}}     
\newcommand{\aanal}{\ensuremath{A\anal}}   
\newcommand{\asynt}{\ensuremath{A^{(s)}}}   

\newcommand{\ljrem}[1]{}

\newtheorem{rem}{Remark}
\newtheorem{thm}{Theorem}

\newtheorem{cor}[thm]{Corollary}

\begin{document}

\begin{frontmatter}


\title{Fully probabilistic design for knowledge fusion between Bayesian filters under uniform disturbances}
%


\author[a1]{Lenka Kukli\v{s}ov\'{a}~Pavelkov\'{a}\corref{c1}}
\author[a1]{Ladislav Jirsa}
\author[a1,a2]{Anthony Quinn}

\date{}

\address[a1]{Czech Academy of Sciences, Institute of Information Theory and Automation  \newline
                Pod vod\'{a}renskou v\v{e}\v{z}\'{\i} 4, Prague, Czech Republic}
\address[a2]{Trinity College Dublin, the University of Dublin, Ireland                }

\cortext[c1]{Corresponding author. E-mail address: pavelkov@utia.cas.cz.}

\begin{abstract}
This paper considers the problem of Bayesian transfer learning-based
knowledge fusion between linear state-space processes driven by
uniform state and observation noise processes. The target task conditions on
probabilistic state predictor(s) supplied by the source filtering
task(s) to improve its own state estimate.  A joint
model of the target and source(s) is not required and is not elicited.
The resulting decision-making problem for choosing the optimal conditional target filtering distribution
under incomplete modelling is solved via fully probabilistic design (FPD),
i.e. via appropriate minimization of Kullback-Leibler divergence (KLD).
The resulting FPD-optimal target learner is robust, in the sense that it can reject poor-quality source knowledge.
In addition, the fact that this Bayesian transfer learning (BTL) scheme does not depend on a model of interaction between the source and target tasks
ensures robustness to the misspecification of such a model. The latter is a problem that affects conventional transfer learning methods.
The properties of the proposed BTL scheme are demonstrated via extensive simulations,
and in comparison with two contemporary alternatives.
\end{abstract}


\end{frontmatter}


\section{Introduction}\label{sec:intro}

Methods of data and information fusion are receiving much attention at present, because of their range of applications in industry 4.0, in the navigation and localization problems, in sensor networks, robotics, and so on \cite{Fou:17}\nocite{DiezEtal:19} -- \cite{AlaHan:20}.

The terms \emph{data} fusion and \emph{information} fusion are often used as synonyms.
However, in some scenarios, the term data fusion is used for raw data that are obtained directly from the sensors, 
while the term information fusion concerns processed or transformed data.
Other terms associated with data fusion include 
data combination, data aggregation, multi-sensor data fusion, and sensor fusion~\cite{Cas:13}.

Data fusion techniques combine multiple sources in order  to obtain improved
(less expensive, higher quality, or more relevant) inferences and decisions compared to a single source. These techniques can be classified into three non-exclusive categories: (i) data association, (ii) state estimation, and (iii) decision fusion \cite{Cas:13}. In this paper, we focus on state estimation methods.

Conventional data fusion methods work with multiple data channels from one common domain, and originating from the same source.
In contrast, cross-domain fusion methods work with data in different domains, but related by a common latent object \cite{Zhe:2015}.
Data from different domains cannot be merged directly. Instead, knowledge---or ``information'' above---has to be extracted from these data and only then fused.
One method of knowledge fusion is transfer learning, also known as knowledge transfer.
This framework aims to extract knowledge from a source domain via a source learning task 
and to use it in the target domain with a given target learning task. The domains or tasks may differ between the source and target \cite{PanYan:10}.
%
%
Examples of successful deployment of transfer learning in data fusion are found in \cite{OuyLow:20}  \nocite{WanZha:16} \nocite{Her:18} -- \cite{ LinHuXiaAlhPir:20}.
In accordance with the DIKW classification scheme proposed in \cite{Bed:20}, we will refer to transfer learning-based fusion as  knowledge fusion.

The performance of transfer learning methods can be improved using computational intelligence
\cite{Leetal:15}.
Bayesian inference provides a consistent approach to building in computational intelligence. It does so via probabilistic uncertainty quantification in decision-making,
taking into consideration the uncertainty associated with model parameters, as well as, the uncertainty associated with combining multiple sources of data.
In the Bayesian transfer learning (BTL) framework---to be championed in this paper---the source and target can be related through a joint prior distribution, as in
\cite{KarQiaDou:18}
%
\nocite{ChaKap:19}
%
--\cite{WanTsu:20}.
%
BTL usually adopts a complete stochastic modelling framework, such as Bayesian networks \cite{LiWanLiWan:20}, Bayesian neural networks \cite{ChaKap:19} or hierarchical Bayesian approaches \cite{WilFerTad:12}.
As already noted, these methods require  a complete model of source-target interaction.
In contrast, in \cite{PapQui:21}, BTL is defined as the task of conditioning a target probability distribution on a transferred source distribution.  A~dual-modeller framework is adopted, where the target modeller conditions on a probabilistic data predictor provided by an independent local source modeller. No joint interaction model between the source and target is specified, and so the source-conditional target distribution is non-unique and can be optimized in this incomplete modelling scenario. The target undertakes this distributional decision-making task optimally, by minimizing an approximate Kullback-Leibler divergence \cite{KulLei:51}. This generalized approach to Bayesian conditioning in incomplete modelling scenario is known as fully probabilistic design~\cite{QuiKarGuy:16}.

Our aim in this paper is to derive a BTL algorithm for knowledge fusion that will use knowledge from several source state-space filters to improve state estimation in a~single target state-space filter. All (observational and modelling) uncertainties are assumed 
to be bounded.
State estimation under bounded noises represents a significant focus for state filtering methods, 
since, in practice, the statistical properties of noises are rarely known,
with only their bounds being available. They avoid the adoption of unbounded noises, that can 
lead to over-conservative design~\cite{Ono:13}.
To the best of our knowledge, the topic of BTL-based multi-task/filter state estimation 
with bounded noises has not yet been addressed in the literature,
except in the author's previous publications \cite{JirPavQui:19a, JirPavQui:20}. 
In those papers, BTL between a~\emph{pair} of filters affected by 
bounded noises is presented. The source knowledge is represented by a bounded output 
(i.e. data) predictor. The optimal target state filtering distribution is 
then designed via FPD. In \cite{JirPavQui:19a}, the support of the state inference 
is an orthotope, while in \cite{JirPavQui:20}, it is relaxed to a parallelotope.

There are fusion techniques for state estimation with bounded noises, but these are conventional fusion methods as defined above.
%
Data fusion methods using set membership  estimation are addressed, for instance, in
\cite{WanSheXiaZhu:19} 
\nocite{SheLiuZhoQinWanWan:19} 
\nocite{Bec:10} 
\nocite{YuaWu:18} 
 -- \cite{XiaYanQin:18}. 
In \cite{HanHor:00}, set membership and stochastic estimation are combined.
In \cite{CheHoYu:17},  local Kalman-like estimates are computed in the presence of bounded noises.
Particle filters \cite{Lietal:16} can also effectively solve the Bayesian estimation problem with bounded noises. However, they are computationally demanding. When used in data fusion context, reduced computational complexity is obtained in \cite{HoaDenHarSlo:15}. 
In \cite{WanSheZhuPan:16} and \cite{BalCaiCri:06}, particle filtering techniques and set membership approaches are combined.

The current paper significantly extends and formalizes results on BTL reported in the above-mentioned authors' papers \cite{JirPavQui:19a} and \cite{JirPavQui:20}.
Both of those papers report an improvement in target performance in the case of concentrated source knowledge (positive transfer) and rejection of diffuse source knowledge (robust transfer). However, the improvement was only minor compared to the performance of the isolated target,
whereas \emph{ad hoc} proposed variants exhibited significantly improved positive transfer.
In the current paper, we formalize the above-mentioned informal variant, showing it to be FPD-optimal. The task of transfer learning-based knowledge fusion with bounded noises is solved in the case where the transferred knowledge is the source's probabilistic state predictor.
An extension to multiple sources is also provided in this paper.

The paper is organized as follows: This section ends with a brief summary of the notation used throughout the paper.
Section \ref{sec:problem} presents the general problem of FPD-optimal Bayesian state inference and estimation in the target, conditioning on transferred knowledge from a~source in the form of the probabilistic state predictor.
In Section \ref{sec:main-res}, these general results are specialized to source
and target state-space models with uniform noises, and
are finally extended to the case of multiple sources in Section \ref{sec:extend}.
Section \ref{sec:experiment-discussion} provides the extensive simulation evidence to illustrate the performance of our FPD-optimal BTL scheme.
Comparison with a contemporary (non-Bayesian) fusion method for uniformly driven state-space models is also provided, as well as comparison with a completely modelled Bayesian network approach.
Section \ref{sec:conclusion} concludes the paper.
The proofs of all the theorems are provided in Appendix~A. 

 \subparagraph*{Notation:}
Matrices are denoted by capital letters (e.g. $A$), vectors and scalars by lowercase
letters (e.g. $b$). $A_{ij}$ is the (i,j)-th element of matrix $A$.  $A_i$ denotes the $i$-th row of $A$.  {$\xsize$ denotes the length of a~(column) vector ${x}$,} and $\mathbb{X}$ denotes
the~set of $x$.  Vector inequalities, e.g. $\underline{x}<\overline{x}$, 
as well as vector maximum and minimum operators, e.g. $\min\{x,y\}$,
are meant entry-wise. $I$ is the identity matrix.  $\chi_{x}(\mathbb{X})$ is
the set indicator, equalling $1$ if $x \in \mathbb{X}$ and $0$ otherwise.
${x}_{t}$ is the value of a~time-variant column
vector ${x}$, at a~discrete time instant, $t \in \mathbb{T} \equiv
\{1,2,\ldots,\overline{t}\}$; $x_{t;i}$ is the $i$-th entry of
${x}_t$; $x(t)\equiv\{x_t, x_{t-1},\ldots, x_1\}$.
$\|x\|_2$ is the Euclidean norm of $x$ and
$\|x\|_{\infty}$ 
is the $H_\infty$ norm of $x$. 
Note that no~notational distinction is made between a~random variable and
its realisation, $X(\omega)\equiv x$. The context will make clear which is meant.

\section{FPD-optimal Bayesian transfer learning (FPD-BTL)} \label{sec:problem}

Assume two stochastically independent modellers, the source (with subscript
\sxn) and the target (without subscript), each modelling their local environment.
Here, we will formulate the task of FPD-optimal Bayesian transfer learning (FPD-BTL) between this source and target, the aim being to improve the target's model of its local environment via transfer of probabilistic knowledge from the source's local environment, as depicted in Figure \ref{fig:transfer-diagram}.

Before addressing the two-task context, let us recall the state estimation problem (filtering) for an \emph{isolated} target, i.e. in the absence of knowledge transfer from a source.

In the Bayesian filtering framework~\cite{Karat:05}, a system of interest is described by the following probability density functions (pdfs):
\begin{eqnarray}\label{eqn:system-pdfs}
\nonumber
&& \mbox{prior pdf } \fu{x_{1}}, \mbox{   observation model } \f{y_{t}}{x_{t}}, t \in \mathbb{T} \\
&& \mbox{and time evolution model } \f{x_{t+1}}{x_{t},u_{t}}, t \in \mathbb{T} \setminus \overline{t}.
\end{eqnarray}
\noindent Here,
${y}_{t}$ is an $\ysize$-dimensional observable output,
${u}_{t}$ is an~optional~$\usize$-dimensional known (exogenous) system input, and
${x}_{t}$ is an $\xsize$-dimensional unobservable (hidden) system state. 
We assume that (i) the hidden state process, $x_{t}$, satisfies the Markov property;
(ii)~no direct relationship between input and output exists in the
observation model; and (iii) the optional inputs constitute of a~known sequence $u_t$, $t \in \mS{t}$, as already stated.

\begin{figure}[t]
  \centering
\includegraphics[width=10cm]{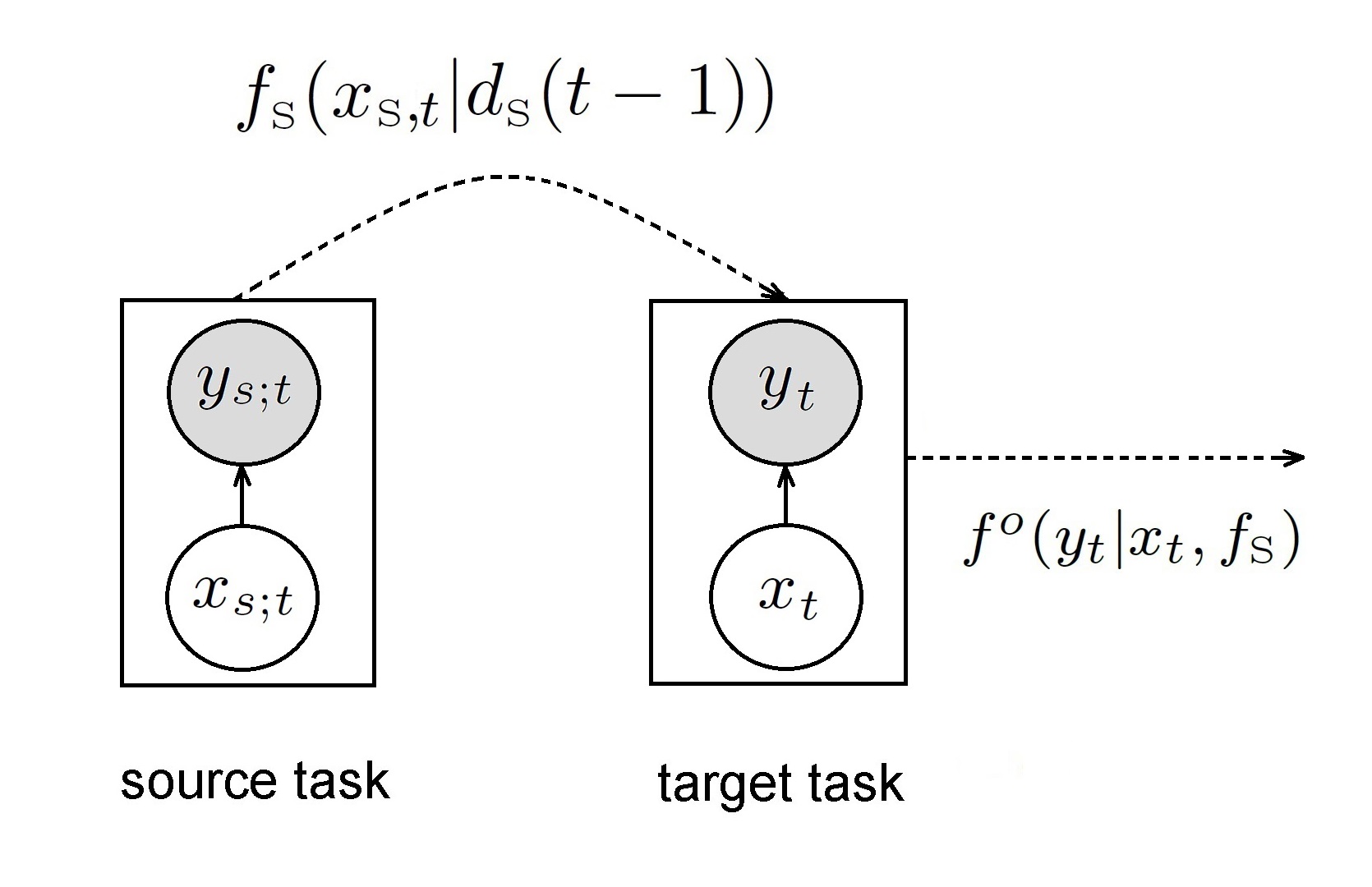}
\caption{Bayesian transfer learning (BTL), involving probabilistic knowledge transfer
from the source to the target Bayesian filter.
The source filter transfers its state predictor,
$f_\sx(x_{\sx,t}|d_\sx(t-1))$, statically at each time, $t=1,2,\ldots,\overline{t}$, to the target
filter to improve the target's filtering performance.
The source data, $d_\sx(t)$, are unobserved by the target,
which computes an optimal conditional model of its local observations,
i.e. $f^o(y_t|x_t,f_\sx)$, at each time, $t$.
Probabilistic knowledge flow is depicted by dashed lines
above.}
\label{fig:transfer-diagram}
\end{figure}

Bayesian filtering, i.e. the inference task of learning the
unknown state process, $x_t$, given the data history $d(t)$,
involves sequential computation of the posterior pdf, $f(x_{t}|d(t))$.
Specifically,
$d(t)$ is a~(multivariate) sequence of observed data,
$d_{t} = \{y_{t}, u_{t}\}$, $t \in \mS{t}$.
Evolution of $f(x_{t}|d(t))$ is described by a~two-step recursion
(the data update and time update) initialized with the prior pdf,
$f(x_1)\equiv f(x_1|d(0))$ \eqref{eqn:system-pdfs}, and ending with
a~data update at the final time, $t=\overline{t}$.

The data update (Bayes' rule) processes the latest datum, $d_t$:
    \begin{equation}\label{eqn:data-updt}
      f(x_{t}|d(t))= \frac{f(y_{t}|x_{t})f(x_{t}|d(t-1))}
      {\int\limits_{\mS{x}_{t}}f(y_{t}|x_{t})f(x_{t}|d(t-1)) \diff x_{t}}.
     \end{equation}

The time update (marginalization) infers the evolution of the
state at the next time:
    \begin{equation}  \label{eqn:time-updt}
   f(x_{t+1}|d(t))=\int\limits_{\mS{x}_{t}}\!f(x_{t+1}|u_{t},x_{t})f(x_{t}|d(t)) \diff x_{t}.
\end{equation}


Next, we return to two stochastically independent modellers,
i.e. the source and target (Figure \ref{fig:transfer-diagram}).
Each filter models its local system,  $\{x_\sx,d_\sx\}$ and $\{x,d\}$,
respectively.
%
The target has access only to the (probabilistic)
state predictor of the source, $f_\sx(x_{\sx,t}|d_\sx(t-1))$, but
not to the actual data or states of the source
(Figure\,\ref{fig:transfer-diagram}).

In the isolated target task, the modeller's complete knowledge
about the evolution of its local state and output
is expressed uniquely by the joint pdf
(i.e. the numerator in \eqref{eqn:data-updt}):
\begin{equation}                      \label{eqn:joint-output-state-pdf}
f(y_{t},x_{t}|d(t-1)) = f(y_{t}|x_{t})\,f(x_{t}|d(t-1)),
\end{equation}
%
where $x_t \in \mS{x}_t$ and $y_t\in\mS{y}_t|x_t$.

Now, performing knowledge transfer as depicted in
Figure~\ref{fig:transfer-diagram}, the target joint
pdf~\eqref{eqn:joint-output-state-pdf} must be conditioned by the
transferred source state predictor, $f_\sx\equiv f_\sx(x_{\sx,t}|d_\sx(t-1))$,
and so the target's knowledge-conditional joint pdf takes the form
$\breve{f}(y_{t},x_{t}|d(t-1),f_\sx)$. Since no joint model of the source and
target relationship is assumed, this pdf is non-unique, and
unknown. Specifically, it is a~variational quantity, $\breve{f}$, in a function set,
$\mS{\breve{f}}$, of possible candidates, as follows:
\begin{equation}                                 \label{e:fac1fac2}
  \breve{f}\equiv\breve{f}(y_{t},x_{t}|d(t-1),f_\sx)
= \breve{f}(y_t|x_t,d(t-1),f_\sx) \, \breve{f}(x_t|d(t-1),f_\sx)
  \in\mS{\breve{f}}
\end{equation}
We now separately examine the two factors on the right-hand side of~\eqref{e:fac1fac2}:
\begin{enumerate}
\item The factor $\breve{f}(x_t|d(t-1),f_\sx)$ represents the target's
  knowledge about its (local) state $x_t$, after transfer
  of the source's state predictor, $f_\sx$,
 to the target. The target chooses to accept the source's
 predictor as its own state model
with \emph{full} acceptance.
  The consequences of this definition will be discussed in Section~\ref{ss:discussion}.
Based on this full acceptance, the target accepts that $x_{\sx,t}$
and $x_t$ are equal in distribution:
\begin{equation}                    \label{e:fac1}
  \breve{f}(x_t|d(t-1),f_\sx) \equiv
  f_\sx(x_{\sx,t}|d_\sx(t-1))\Big|_{x_{\sx,t}\rightarrow x_t} =
  f_\sx(x_{t}|d_\sx(t-1)).
\end{equation}
In consequence, the factor  \eqref{e:fac1} is fixed in~\eqref{e:fac1fac2}.
\item The factor $\breve{f}(y_t|x_t,d(t-1),f_\sx)$ now remains
as the only variational factor,
being a consequence of the
target's choice not to elicit an interaction model
between the source and target (Figure \ref{fig:transfer-diagram}).
According to~\eqref{eqn:system-pdfs}, $f(y_t|x_t,d(t-1))=f(y_t|x_t)$, i.e. the observation model is
conditionally independent, given $x_t$, of $d(t-1)$. This conditional independence is preserved by the knowledge transfer.
Therefore,
\begin{equation}                    \label{e:fac2}
  \breve{f}(y_t|x_t,d(t-1),f_\sx) \equiv \breve{f}(y_t|x_t,f_\sx).
\end{equation}
The main design---i.e. decision---problem for the target is now to
choose an optimal form of  \eqref{e:fac2}.
\end{enumerate}
Inserting \eqref{e:fac1} and~\eqref{e:fac2} into \eqref{e:fac1fac2}:
\begin{equation}                        \label{eqn:joint-output-state-pdf-constr}
 \breve{f} \equiv  \breve{f}(y_{t},x_{t}|d(t-1),f_\sx) =
  \breve{f}(y_{t}|x_{t},f_\sx)\
  f_\sx(x_{t}|d_\sx(t-1)),
\end{equation}
where $x_{t} \in \mS{x}_{\sx,t}$ and $y_{t}\in\mS{y}_{t}|x_{t},f_\sx$.
%
The set, $\breve{\mS{f}}$, of the target's admissible joint models,
$\breve{f}$, following knowledge transfer from the source, is therefore
\begin{equation}      \label{eqn:set-of-models}  
  \breve{f} \in \breve{\mS{f}}
  \equiv
  \{\mbox{set of models}\ \eqref{eqn:joint-output-state-pdf-constr} \mbox{ with } f_\sx(x_t|\cdot)
\mbox{ fixed and } \breve{f}(y_t|\cdot)
\mbox{ variational}\}.
\end{equation}

The optimal pdf, ${f}^o(y_t|x_t,f_\sx)\in\breve{\mS{f}}$, respecting both the
transferred knowledge and the target filter behaviour, is sought using
fully probabilistic design (FPD), which is an axiomatically justified
procedure for distributional decision-making~\cite{KarKro:12}, \cite{Ber:79}.
It seeks $\breve{f} = {f}^o \in \breve{\mS{f}}$, being the joint
pdf~\eqref{eqn:joint-output-state-pdf-constr}
that minimizes the Kullback-Leibler divergence (KLD)
(below)~\cite{KulLei:51} from $\breve{f}$ to the target's fixed ideal, $f^{I}$.
This ideal is defined as~\eqref{eqn:joint-output-state-pdf}, i.e.
the joint pdf of the isolated target filter, modelling its behaviour
prior to (i.e. without) the transfer of source knowledge.
To summarize, the ideal pdf, and the knowledge-conditional pdf to be
designed by the target are, according to~\eqref{eqn:joint-output-state-pdf},
\eqref{eqn:joint-output-state-pdf-constr} and~\eqref{eqn:set-of-models}:
\begin{eqnarray}
    \mbox{ideal:}\hspace{0.6em} f^{I} &\equiv&  f(y_{t}|x_{t})\ f(x_{t}|d(t-1)), \label{e:ideal}\\
	\mbox{variational:}\hspace{1em} \breve{f} &\equiv&
    \underbrace{\breve{f}(y_{t}|x_{t},f_\sx)}_{\mbox{\scriptsize to be optimized}} \label{e:optimized}
    f_\sx(x_{t}|{d_\sx}(t-1)).
\end{eqnarray}
%
%
Recall that the KLD \cite{KulLei:51} from $\breve{f}$ to $f^{I}$
is defined as
\begin{equation}                                         \label{e:klddef}
  \kld{\breve{f}}{f^{I}} = \mathsf{E}_{\breve{f}}\left[\ln\frac{\breve{f}}{f^{I}}\right],
\end{equation}
where $\mathsf{E}_{\breve{f}}$ denotes expectation with respect to $\breve{f}$.
%
FPD consists in minimizing this KLD~\eqref{e:klddef} objective as a function of
$\breve{f}(y_{t}|x_t,f_\sx)$ in~\eqref{e:optimized}, for the fixed ideal~\eqref{e:ideal}, i.e.
\begin{equation}                                \label{e:fpdoptim-wild}
  {f}^o(y_{t}|x_{t},f_\sx) \equiv
  \arg\mathop{\min}\limits_{\breve{f}(y_{t}|x_t,f_\sx)} \kld{\breve{f}}{f^I},
\end{equation}
\eqref{e:fpdoptim-wild} conditions the target's knowledge about
future $y_{t}$ on the transferred $f_\sx$
in an  FPD-optimal manner. For simplicity, the superscript $^o$ will
be omitted in the resulting FPD-optimal pdf,
i.e. ${f}(y_{t}|x_{t},f_\sx) \equiv  {f}^o(y_{t}|x_{t},f_\sx)$.

\pagebreak[2]
\noindent
We note the following: \nopagebreak[4]
\begin{itemize}
\item The transferred source knowledge, $f_\sx(x_{\sx,t}|d_\sx(t-1))$,
can be elicited in various ways that are unknown to
the target; e.g. as an empirical distribution
of a quantity similar to $x_t$, or
some unspecified distributional approximation, etc. \cite{QuiKarGuy:17}.
In this paper, involving multiple state filtering tasks, we will assume
that $f_\sx(x_{\sx,t}|d_\sx(t-1))$ is the output of the source's synchronized
time update at $t-1$ \eqref{eqn:time-updt}.
\item In the authors' previous publications \cite{JirPavQui:19a}, \cite{JirPavQui:20}, \cite{FolQui:18},
it was the source \emph{data} predictor which was transferred. Instead, here,
\emph{for the first time}, it is the source state predictor,
$f_\sx(x_{\sx,t}|d_\sx(t-1)), t \in \mS{t}$, that is transferred.
As we will see later, \emph{this setting ensures robust knowledge transfer}.

\end{itemize}

Recall that our aim is to specialize FPD-optimal Bayesian transfer
learning (FPD-BTL framework) defined in~\eqref{e:ideal}, \eqref{e:optimized},
\eqref{e:fpdoptim-wild} to a pair of Bayesian filters under \emph{bounded}
observational and state noises. We now address this aim.

\section{FPD-BTL between LSU-UOS filtering tasks} \label{sec:main-res}

As noted in Section~\ref{sec:intro}, we are specifically interested in
knowledge processing among interacting Bayesian state-space
filters with uniform noises (LSU models, see below). We
therefore instantiate the FPD-optimal scheme~\eqref{e:fpdoptim-wild}
for conditioning the target's observation model on the source's
transfered state predictor in this specific context.
%
Firstly, in Section~\ref{sec:prelim}, we
review the isolated LSU-UOS filter,
and derive the approximate solution to the related state estimation problem.
Then, the required instantiation of FPD-BTL to a pair of these LSU-UOS
filters is presented in Section~\ref{sec:main}.  In
Section~\ref{sec:extend},
the framework is extended to
multiple LSU-UOS source filters, transferring probabilistic state
knowledge to a single target.


\subsection{LSU-UOS filtering task for the isolated target} \label{sec:prelim}

The  general stochastic system description in~\eqref{eqn:system-pdfs} is now
instantiated as a linear state-space model \cite{Sod:02}
\begin{alignat}{2}
     \label{eqn:SS-model-observ}
     y_{t} &=  Cx_t  + v_{t}           &&   \equiv \tilde{y}_t + v_t, \\
     \label{eqn:SS-model-state-evo}
     x_{t+1}  &=  Ax_t + Bu_t  + w_{t+1}  &\ & \equiv \tilde{x}_{t+1} + w_{t+1},
\end{alignat}
where $x_t\in\mS{r}^{\xsize}$, $y_t\in\mS{r}^{\ysize}$,
$u_t\in\mS{r}^{\usize}$.  $A$, $B$, $C$ are known model matrices of
appropriate dimensions; $v_t$ and $w_{t}$ are additive
random processes expressing observational and modelling uncertainties,
respectively, and their stochastic model must now be specified.
We assume that $v_t$ and $w_{t}$ are mutually independent
white noise processes uniformly distributed on \textit{known} supports of finite measure:
%
%
\begin{equation}
  f(v_t)=\mathcal{U}_v(-\ynparv, \ynparv),
  \hspace{2em}f(w_t)=\mathcal{U}_w(-\xnparv, \xnparv), \label{eqn:SS-model-noises-unif}\\
\end{equation}
where $\xnparv\in\mS{r}^{\xsize}$, $\ynparv\in\mS{r}^{\ysize}$, with finite
positive entries, and $\mathcal{U}_{\bullet}$ denotes the uniform
pdf on an orthotopic support (UOS), as now defined.

\begin{rem}\label{rem:UOS-def}
Consider a~finite-dimensional vector random variable, $z \in \mS{Z}_\mS{o}$, with
realisations in the following bounded subset of $\mS{r}^{\zsize}$:
\begin{equation}                                          \label{eqn:ortho-set}
\mS{z}_\mS{o} \equiv \{z: \lo{z} \leq  z \leq  \hi{z}\}, \ \ \lo{z} <  \hi{z}
\end{equation}
where $\lo{z},\,\hi{z}\in \mS{r}^{\zsize}$.
This convex polytope,  $\mS{Z}_\mS{o}$, is called an \emph{orthotope}.

The~uniform pdf of $z$ on the orthotopic support~\eqref{eqn:ortho-set} called the \emph{UOS pdf} is
defined as
\begin{equation}                                   \label{eqn:unif-notation-ortho}
  \mathcal{U}_z(\lo{z},\hi{z})\equiv
  \mathcal{V}^{-1}_O \chi_z\left(\lo{z} \leq z   \leq  \hi{z}\right),
\end{equation}
where $\mathcal{V}_O=
\prod\limits_{i=1}^{\zsize}(\hi{z}_i-\lo{z}_i)$.
\end{rem}

Model \eqref{eqn:SS-model-observ}, \eqref{eqn:SS-model-state-evo},
\eqref{eqn:SS-model-noises-unif}, together with
\eqref{eqn:unif-notation-ortho}, defines the linear state-space mode
with uniform additive noises on orthotopic supports, denote the
LSU-UOS model.  Its observation and state evolution models
\eqref{eqn:system-pdfs} are equivalently specified as
\begin{align}   
f(y_t|x_t)&\equiv\mathcal{U}_y({\tilde{y}_t}-\ynparv, {\tilde{y}_t}+\ynparv) \label{e:LSU-dupdt}\\
f(x_{t+1}|x_{t},u_{t})&\equiv\mathcal{U}_x({\tilde{x}_{t+1}} - \xnparv,
{\tilde{x}_{t+1}} + \xnparv).                                      \label{e:LSU-tupdt}
\end{align}

Exact Bayesian filtering for the LSU model~\eqref{e:LSU-dupdt} and~\eqref{e:LSU-tupdt}%
---i.e. computation of $f(x_{t}|d(t))$ following~\eqref{eqn:data-updt} and~\eqref{eqn:time-updt}---is intractable, since
the UOS class of pdfs (Remark \ref{rem:UOS-def}) is not
closed under those filtering operations.  One consequence is that the
dimension of the sufficient statistic of the filtering pdf
\eqref{eqn:data-updt} is unbounded as $t$ grows i.e. at an infinite filtering horizon and so cannot be implemented (the curse of dimensionality~\cite{Karat:05}).  In~\cite{JirPavQui:19b,PavJir:18},
approximate Bayesian filtering with the LSU
model~\eqref{e:LSU-dupdt} and~\eqref{e:LSU-tupdt},
closed within the UOS class \eqref{eqn:unif-notation-ortho}, is proposed.
This involves a local approximation after each data update
\eqref{eqn:data-updt} and time update \eqref{eqn:time-updt}, as
recalled below.  This tractable but approximate Bayesian filtering
procedure will be called \emph{LSU-UOS Bayesian filtering}.

\subsubsection{LSU-UOS data update}\label{sec:data-updt-UOS}

Define a~\emph{strip}, $\mS{z}_{\mS{S}}$, as
a set in $\mS{r}^{\zsize}$ bounded by two parallel hyperplanes,
as follows:
\begin{equation}                                               \label{eqn:strip}
 \mS{z}_{\mS{S}} = \{z: a \leq c'\,z \leq b\}.
\end{equation}
Here $a < b$ are scalars, and $c\,\in \mS{r}^{\zsize}$.

In the data update \eqref{eqn:data-updt}, prior $f(x_{t}|d(t-1))=\mathcal{U}_x(\underline{x}^{+}_t,\overline{x}^{+}_t)$
is processed together with $f(y_t|x_{t})$ in~\eqref{e:LSU-dupdt}, and with the latest observation, $y_t$,
via Bayes' rule.  It starts at $t\!=\!1$ with
$f(x_{1})=\mathcal{U}_x(\underline{x}^{+}_1,\overline{x}^{+}_1)$.  The
resulting filtering pdf is uniformly distributed on a~polytopic support that results
from the intersection of~the~ortho\-topic support of $f(x_{t}|d(t-1))$ 
and $\ysize$ strips induced by the latest observation, $y_t$: 
$$
f(x_{t}|d(t)) \propto
\mathcal{U}_x(\underline{x}^{+}_t,\overline{x}^{+}_t)\
\mathcal{U}_x(y_{t}-\ynparv \leq Cx_{t} \leq y_{t}+\ynparv)
\propto$$

\begin{equation}                                \label{eqn:data-updt-unif2}
\propto  \chi_{x} \left(
 \left[
   \begin{array}{c}
     \underline{x}^{+}_t \\
     y_{t} -\ynparv \\
   \end{array}
 \right]
 \leq
 \left[
   \begin{array}{c}
     I \\
     C \\
   \end{array}
 \right]
 x_t
 \leq
 \left[
   \begin{array}{c}
     \overline{x}^{+}_t \\
     y_{t} +\ynparv \\
   \end{array}
 \right]
  \right).
\end{equation}

\noindent In \cite{JirPavQui:19b}, a local approximation is proposed so that the
resulting polytopic support of \eqref{eqn:data-updt-unif2} is
circumscribed by an orthotope, giving
\begin{equation}\label{eqn:data-updt-UOS-approx}
f(x_{t}|d(t))\approx \mathcal{U}_{ x_{t}}(\underline{x}_t, \, \overline{x}_t).
\end{equation}
The approximate Bayesian sufficient statistics, $\underline{x}_t$ and
$\overline{x}_t$, process $d(t)$ tractably $\forall t$, yielding an implementable algorithm.
The details are provided
in \cite{JirPavQui:19b}.

\subsubsection{LSU-UOS time update}\label{sec:time-updt-UOS}
It now remains to ensure that each data update (above) is, indeed,
presented with a UOS output from the preceding time update, as
presumed.  In each time update, the 
UOS posterior, $f(x_{t}|d(t))$
\eqref{eqn:data-updt-UOS-approx}, is processed together with
$f(x_{t+1}|x_t,u_t)$ \eqref{e:LSU-tupdt}---uniform on 
$\xsize$ $x_t$-dependent strips---via the marginalization
operator in \eqref{eqn:time-updt}. The resulting pdf does have an orthotopic 
support, but is not uniform on it.
In \cite{JirPavQui:19b}, the following local
approximation projects $f(x_{t+1}|d(t))$ back into the UOS class, $\forall t$:
%
%
\begin{equation}                                          \label{eqn:time-updt-UOS-approx}
f(x_{t+1}|d(t))\approx \mathcal{U}_{x_{t+1}}(\underline{x}^{+}_{t+1},\overline{x}^{+}_{t+1})
\end{equation}
where
%
%
\begin{equation}                                    \label{eqn:m-computation}
  \lo{x}^{+}_{i,t+1}=\sum_{j=1}^{\xsize}\min\{A_{ij}\lo{x}_{t;j}+B_i u_{t},A_{ij}\hi{x}_{t;j}+B_i u_{t}\}-\xnparv_i
\end{equation}
$$
~~\hi{x}^{+}_{i,t+1}=\sum_{j=1}^{\xsize}\max\{A_{ij}\lo{x}_{t;j}+ B_i u_{t},A_{ij}\hi{x}_{t;j}+B_i u_{t}\}+\xnparv_i,
$$
$i=1,\ldots,\xsize$.

\subsection{FPD-BTL between a pair of LSU-UOS filtering tasks} \label{sec:main}

We now return to the central concern of this paper: the static
FPD-optimal transfer of the state predictor, $f_\sx(x_{\sx,t}|d_\sx(t-1))$,
from the source LSU-UOS filter (``the source task'') to the target
LSU-UOS filter (``the target task''). The transfers will occur
\emph{statically}, $\forall t \in\mS{t}$, meaning that the
\emph{marginal} state predictor, $f_\sx(x_{\sx,t}|d_\sx(t-1))$ is transferred
in each step of FPD-BTL. (For a derivation of joint source knowledge
transfer---i.e. dynamic transfer---in Kalman filters, see~\cite{PapQui:18}.)

Although there exists an explicit functional minimizer
of~\eqref{e:fpdoptim-wild} (see.~\cite{QuiKarGuy:17}), our
specific purpose here is to instantiate this FPD-optimal solution
for UOS-closed filtering in the source and target tasks, as defined
in Section~\ref{sec:prelim}.

We propose that the FPD-optimal target knowledge-constrained
observation model~\eqref{e:fpdoptim-wild} (i.e. after transfer),
$f^o\equiv f(y_t|x_t,f_\sx)$, be uniform with its support,
$\breve{\mS{y}}_t|x_t$, bounded (here, our set notation emphasizes
the fact that the support is a~function of $x_t$).
We now prove that this choice is closed under the FPD
optimization~\eqref{e:fpdoptim-wild}.
While the following theorem is formulated for uniform pdfs on
\emph{general} bounded sets, it is applied to our UOS class in the sequel.

\begin{thm}                       \label{t:t1a}
  Let the target's ideal pdf in FPD~\eqref{e:fpdoptim-wild}
  be its isolated joint predictor~\eqref{e:ideal}. Assume
  that the target's (pre-transfer) state predictor,
  $f(x_t|d(t-1))$ is uniform on bounded support, $\mS{x}_t$.
  $f(y_{t}|x_{t})$ is defined in~\eqref{e:LSU-dupdt}.
The transferred source state predictor,
$f_\sx(x_{\sx,t}|d_\sx(t-1))$,
is also uniform, with bounded support,
$\mS{x}_{\sx,t}$.
Define the bounded intersection (Figure~\ref{f:venn12}):
\begin{equation}                             \label{e:xconstr}
\mS{x}^\cap_{t} = \mS{x}_{\sx, t} \cap \mS{x}_t.
\end{equation}
%
Assume that the (unoptimized) variational target observation model,
$\breve{f}(y_t|x_t,f_\sx)$~\eqref{e:optimized}, is also
uniform with bounded support.

If $\mS{x}^\cap_{t}\neq\emptyset$, then the optimal choice
of $\breve{f}(y_t|x_t,f_\sx)$ minimizing the
FPD objective~\eqref{e:fpdoptim-wild} is
\begin{equation}                                       \label{e:fpdoptim-alt}
  f(y_{t}|x_{t}, f_\sx) \propto   f(y_{t}|x_{t}\in \mS{X}^{o}_{t})
\end{equation}
where the FPD-optimal set of $x_t$ after transfer of the source knowledge is
deduced to be $\mS{x}^{o}_{t}=\mS{x}^\cap_{t}$.

If $\mS{x}^\cap_{t}=\emptyset$---a testable condition \emph{before}
transfer---then knowledge transfer is
stopped,\footnote{This decision is consistent with the definition of
  conditional probability.} and $f(y_t|x_t,
f_\sx)\equiv f(y_t|x_t)$, i.e. the optimal target conditional
observation model is defined to be
that of the isolated target.
\end{thm}

\begin{proof}
  See~\ref{a:proof1}
\end{proof}
%
\begin{figure}[ht]
\hspace*{-1.5em}
\bgroup
\def\arraystretch{0.3}%
\begin{tabular}{cc}
 {\includegraphics[height=0.33\textwidth]{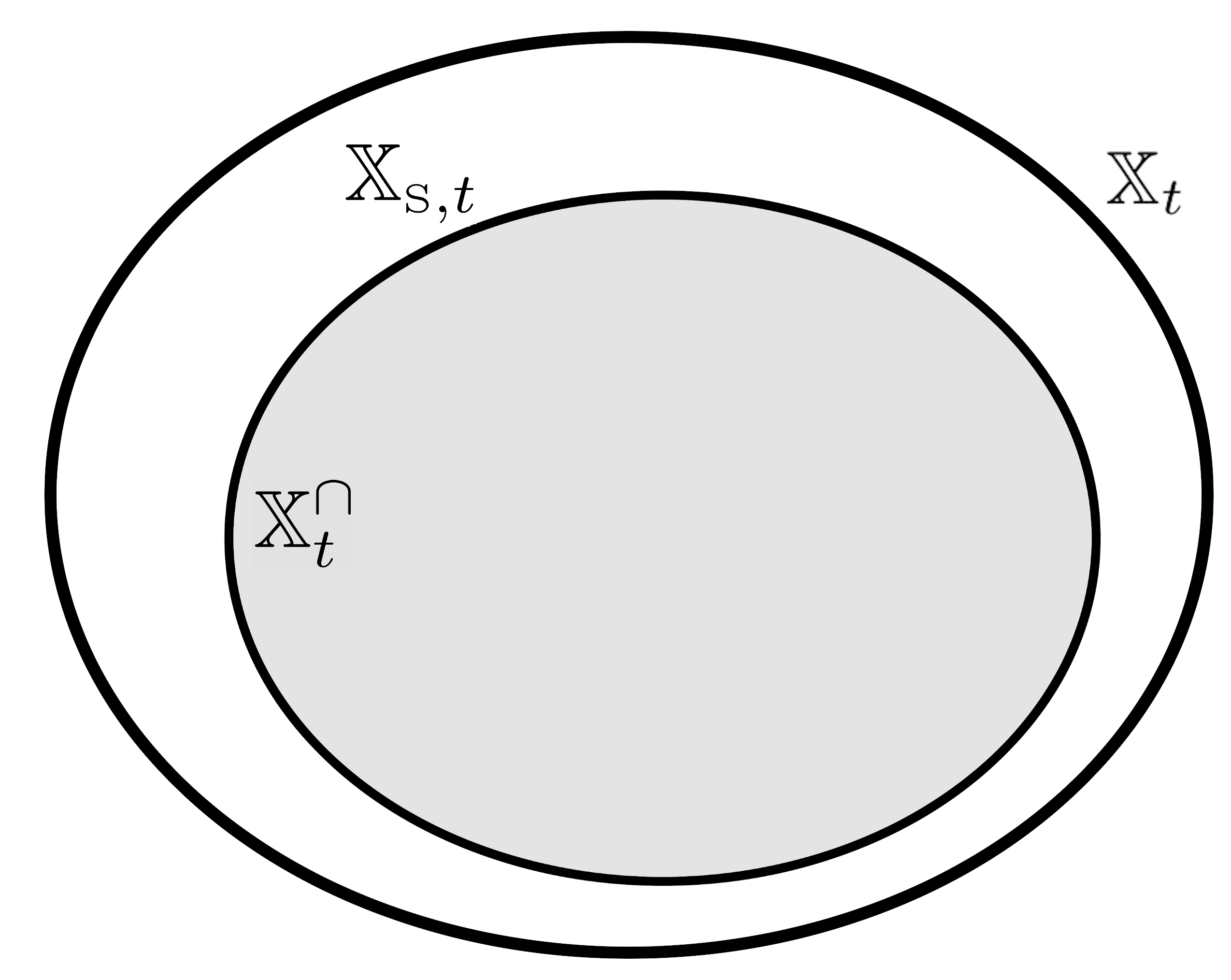}} &
 {\includegraphics[height=0.33\textwidth]{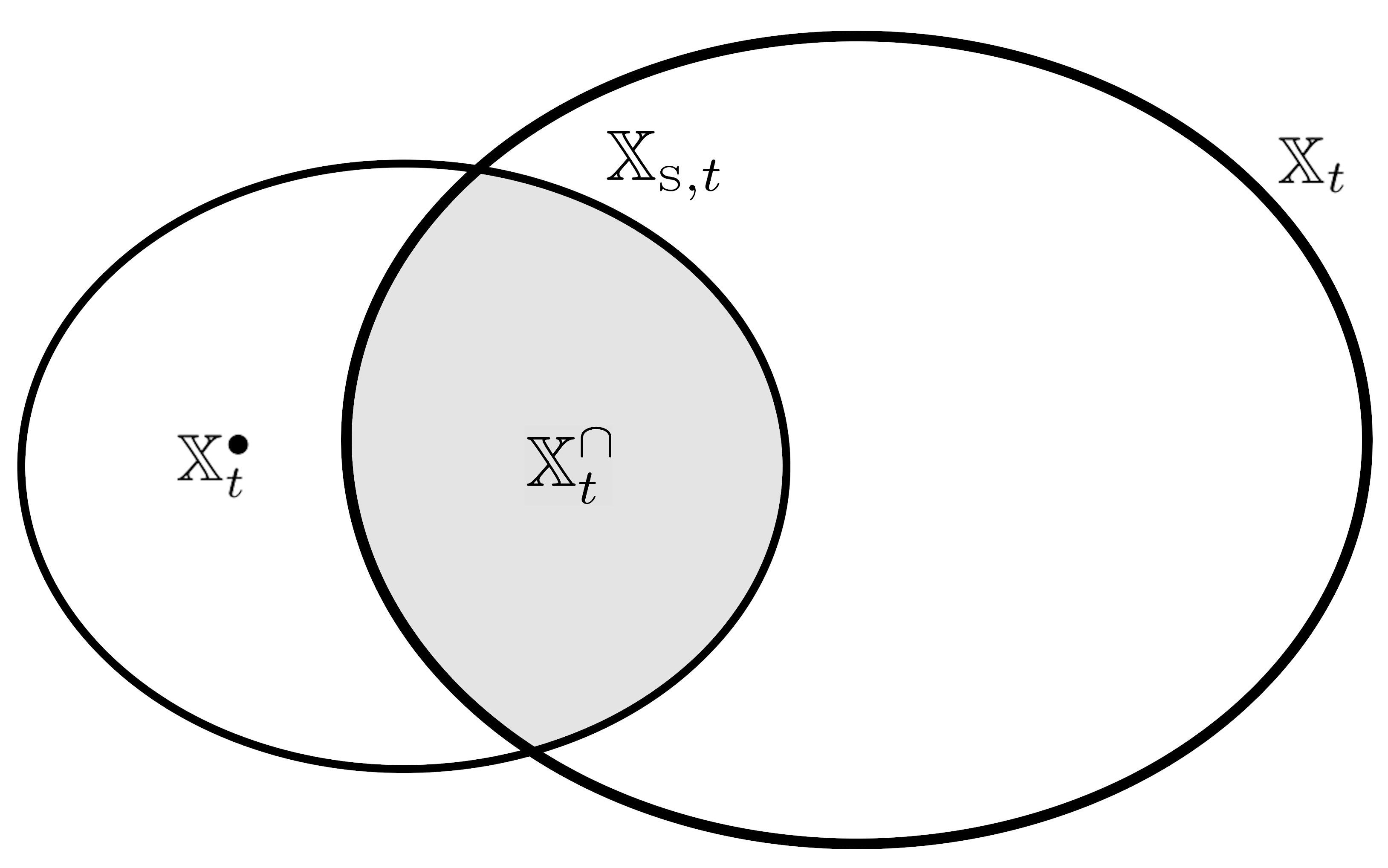}} \\
case~\ref{itemi} & case~\ref{itemii} \\
~\\~\vphantom{{\LARGE B}}\\
 \multicolumn{2}{c}{\includegraphics[height=0.33\textwidth]{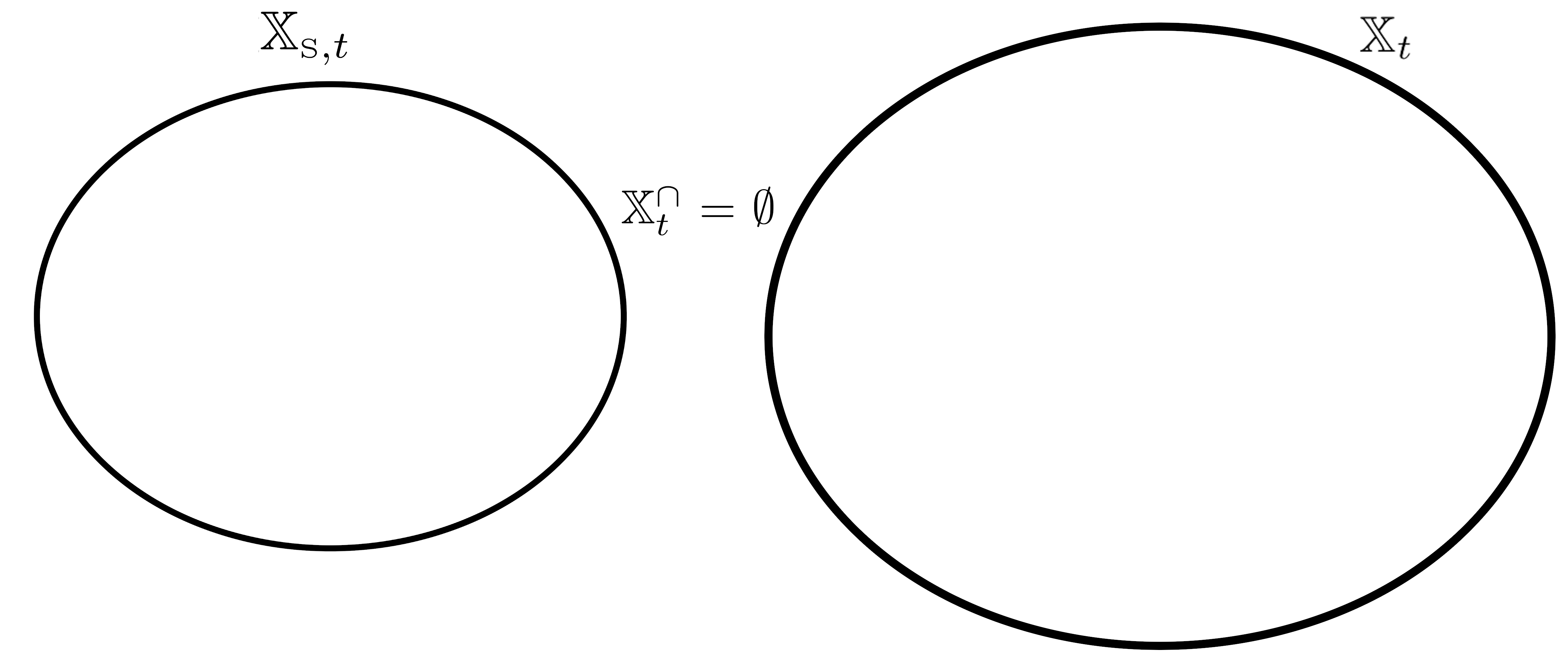}} \\
 \multicolumn{2}{c}{case~\ref{itemiii}}
\end{tabular}
\egroup
 \caption{The mutual positions of the supports, $\mS{x}_{\sx,t}$ of
   $f_\sx(x_{_\sx,t}|d_\sx(t-1))$, and $\mS{x}_{t}$ of
   $f(x_{t}|d(t-1))$. The cases~\ref{itemi}, \ref{itemii}
   and~\ref{itemiii} are
   separately considered in the
   proof of
   Theorem\,\ref{t:t1a}.
}
 \label{f:venn12}
\end{figure}

The sets $\mS{x}_t$ and $\mS{x}_{\sx, t}$ are functions only of
$d(t-1)$ and $d_\sx(t-1)$, respectively, i.e. they are local \emph{statistics} of
the target or source tasks, respectively. In this way, FPD-BTL is
effecting transfer of optimal statistics  (knowledge) from source to target,
in the spirit of knowledge fusion (Section~\ref{sec:intro}).
This is in contrast to any requirement to transfer raw data from
the source, for processing in the target, as occurs in
conventional multi-task inference
(see Section~\ref{sec:experiment-discussion}). This property of transfer
of source-optimal statistics to the target is a defining characteristic
of FPD-BTL.

\begin{cor}[Specialization to the UOS case] \label{cor:spec2UOS}
Orthotopic sets~\eqref{eqn:ortho-set} are closed under
the intersection operator~\eqref{e:xconstr} (if $\mS{x}^\cap_t \neq\emptyset$).
Specifically, if
$\mS{x}_{\sx, t}=\left\{x: \lo{x}_{\sx,t}\leq x \leq \hi{x}_{\sx,t} \right\}$
and $\mS{x}_{t}=\left\{x: \lo{x}_{t}\leq x \leq \hi{x}_{t} \right\}$,
then
the FPD-optimal set of $x_t$ after transfer~\eqref{e:fpdoptim-alt} is
\begin{equation}                                 \label{eqn:spec2UOS}
\mS{x}^{o}_t = \mS{x}^\cap_t =
\left\{x:\max(\lo{x}_{\sx,t},\lo{x}_{,t})\leq x \leq
\min(\hi{x}_{\sx,t},\hi{x}_{,t}) \right\}.
\end{equation}
%
\end{cor}
\begin{cor} \label{cor:dupd-optim}
  \eqref{e:fpdoptim-alt} constraints the allowed set of states in the
  target's subsequent (i.e. post-transfer) processing of local
  datum, $d_t$, via data update~\eqref{eqn:data-updt}.
  The latter can be written as
\begin{eqnarray}                                  \label{e:dupd-optim}
f(x_t|d(t),f_\sx) &\propto&
f(y_{t}|x_t, f_\sx)
\,f(x_{t}|d(t\!-\!1)) \propto
\nonumber \\
&\propto& f(y_{t}|x_{t})
\underbrace{
    f(x_{t}|d(t\!-\!1))
    \,
  \chi\!\left(x_{t}\in \mS{x}^{o}_{t}\right)
}_{\propto\ f(x_{t}|d(t-1), f_\sx)}.
\end{eqnarray}
Effectively, then the FPD-optimal transfer restricts the support of the
target's (prior isolated) state predictor, $f(x_t|d(t-1))$, to
$\mS{x}^\cap_t$~\eqref{e:xconstr}, and this then forms the prior for
the subsequent processing of the target's local datum, via a
conventional data update~\eqref{e:dupd-optim}.
\end{cor}

\noindent Additional notes:
\begin{itemize}
\item The knowledge is processed sequentially in the target, i.e.
\emph{firstly} the target processes local
  $f_\sx(x_{_\sx,t}|d_\sx(t-1))$, yielding
  $f(y_{t}|x_t, f_\sx)$),~\eqref{e:fpdoptim-alt};
  \emph{secondly}, the target filter
  processes local $y_t$ (data update~\eqref{e:dupd-optim});
  \emph{thirdly}
  the target predicts $x_{t+1}$ via the \emph{local} target time
  update~\eqref{eqn:time-updt},
  making available to the next (3-part) step of FPD-BTL its
  knowledge-conditional state predictor, $f(x_{t+1}|d(y), f_\sx)$.
  Knowledge transfer is therefore interleaved between the time and data updates.
\item The FPD-optimal intersection~\eqref{e:xconstr}, \eqref{e:fpdoptim-alt}
is a~concentration operator in the inference scheme (see Figure~\ref{f:venn12}),
ensuring entropy reduction and consistency properties~\cite{Vaa:98}
which---though evident---are not proven here.
\item Recall the full acceptance of the source's state predictor by the
  target, this induces a~discontinuity between
  $\mS{x}^\cap_t\neq\emptyset$ and $\mS{x}^\cap_t\equiv\emptyset$
  in Theorem~\ref{t:t1a} (see Figure~\ref{f:venn12}). This artefact will be
  discussed further in Section~\ref{ss:discussion}.
\end{itemize}

\noindent The implied algorithmic sequence for FPD-BTL between
a~pair of LSU-UOS filters is provided in Algorithm~\ref{alg:alg}.

\begin{algorithm}
  \hspace{-1.5em}
\textbf{Initialization:}
\begin{itemize}[itemsep=-0.7em]
\item [-] set the initial time $t=1$ and the final time $\hi{t}>1$
\item [-] set prior values $\lo{x}^{+}_1$,
      $\hi{x}^{+}_1$
      for $f(x_1|d(0))\equiv f(x_1) =
      \mathcal{U}_x(\lo{x}^{+}_1,\hi{x}^{+}_1)$
\item [-] set
      $f(x_{\sx,1}|d_\sx(0))\equiv f_\sx(x_{\sx,1}) =
    \mathcal{U}_{\sx}(\lo{x}^{+}_{\sx,1},\hi{x}^{+}_{\sx,1})$
\item [-] set noise bounds $\ynparv$, $\xnparv$ \eqref{eqn:SS-model-noises-unif}
\end{itemize}
  \hspace{-1.5em}
\textbf{Recursion:}
\For{$t = 1, \ldots, \hi{t}-1$}
    {
      \vspace{-0.5em}
  \begin{itemize}
    \item[I.]  \emph{Knowledge transfer:}\\
    transfer orthotopic $f_\sx(x_{\sx,t}|d_\sx(t-1))$ \eqref{e:fac1}\\ and compute
    $f(x_{t}|d(t-1),f_\sx)$ \eqref{e:dupd-optim} via \eqref{e:xconstr} and \eqref{eqn:spec2UOS}
    \item[II.] \emph{Data update:}\\
      process local target datum, $d_t$,
      into $f(x_{t}|d(t))$ \eqref{eqn:data-updt-UOS-approx}
      via orthotopic approximation of \eqref{eqn:data-updt-unif2},
      specified in \cite{JirPavQui:19b}
    \item[III.] \emph{Time update:}\\
    compute $f(x_{t+1}|d(t)) = \mathcal{U}_x(\lo{x}^{+}_{t+1},\hi{x}^{+}_{t+1})$~\eqref{eqn:time-updt-UOS-approx}
      via \eqref{eqn:m-computation}
  \end{itemize}
  }

  \hspace{-1.5em}
\textbf{Termination:} set $t=\hi{t}$
      \vspace{-1em}
  \begin{itemize}
  \item[I.]  \emph{Knowledge transfer:}\\
    transfer final orthotopic
      $f_\sx(x_{\sx,\hi{t}}|d_\sx(\hi{t}-1))$ \eqref{e:fac1}\\ and
      compute $f(x_{\hi{t}}|d(\hi{t}-1),f_\sx)$ \eqref{e:dupd-optim}
      via \eqref{e:xconstr} and \eqref{eqn:spec2UOS}
    \item[II.] \emph{Data update:}\\
      process final local target datum,
      $d_{\hi{t}}$, into $f(x_{\hi{t}}|d(\hi{t}))$
      \eqref{eqn:data-updt-UOS-approx} via orthotopic approximation of
      \eqref{eqn:data-updt-unif2} \cite{JirPavQui:19b}
  \end{itemize}
\caption{FPD-BTL between two LSU-UOS filtering tasks}
\label{alg:alg}
\end{algorithm}

\subsection{FPD-BTL for multiple LSU-UOS sources and a single LSU-UOS target} \label{sec:extend}

Here, we extend FPD-BTL (Sec.~\ref{sec:main}) to the case of multiple
bounded-support sources, which can be specified to the case of
multiple interacting LSU-UOS tasks, again via
Corollaries~\ref{cor:spec2UOS} and~\ref{cor:dupd-optim}.  Assume the
same scenario as in Figure \ref{fig:transfer-diagram}, with one target
but, now, $n-1$ sources, $n\geq 2$ (i.e.\ $n\geq 2$ interacting
LSU-OUS tasks in total).
Once again, the instantiation of the $n$ tasks is avoided
(i.e.\ incomplete modelling). Each source provides its state predictor
$f_{\sx_{i}}(x_{\sx_{i},t}|d_{\sx_{i}}(t-1))$, $i=1,\ldots,n-1$,
statically, $\forall t$, to the target in the same way as in the
single source setting.
\begin{thm}                                  \label{t:t2}
  Let there be $n\geq 2$ state-space filters, $f_1$, $f_2$, \ldots,
  $f_{n}$, having bounded supports of their state predictors,
  $\mS{x}_{1,t}$, $\mS{x}_{2,t}$,\ldots, $\mS{x}_{n,t}$, respectively.
  Assume $f_1$ is the~target filter, and $f_2$, \ldots, $f_{n}$ the
  source filters.
%
%
Then the FPD-optimal target observation model after transfer for the
$n-1$ source state predictors is
$$f(y_t|x_t, f_2,\ldots,f_n)=f(y_{t}|x_{t}\in \mS{X}^{o}_{t}),$$
where
  \begin{equation}                                   \label{e:xconstrnet}
    \mS{x}^{o}_t = \bigcap\limits_{k=1}^{n} \mS{x}_{k,t}. 
  \end{equation}
%
\end{thm}

\begin{proof}
See~\ref{a:proof2}.
\end{proof}

\section{Simulations studies}         \label{sec:experiment-discussion}

In this section, we provide a detailed study of the performance of the
proposed Bayesian transfer learning algorithm (FPD-BTL) between
LSU-UOS filtering tasks. We compare it to Bayesian complete (network)
modelling (BCM, to be defined below) for the UOS class and to the
distributed set-membership fusion algorithm (DSMF) for ellipsoidal
sets~\cite{WanSheXiaZhu:19}, which also involves complete modelling of
the networked LSU filters.

In the design of these comparative experiments, our principal concerns
are the following:
\begin{enumerate}
\item To study the influence of the number of sources on the
  performance of the target filter in FPD-BTL (experiment \#1).
\item To compare FPD-BTL to complete modelling alternatives (BCM and
  DSMF, experiment \#2).
\item To study the robustness of FPD-BTL---which does not require for
  tasks interaction (i.e.\ it is incompletely modelled)---to model
  mismatches that inevitably occur between source and target tasks in
  the complete modelling approaches (BCM and DSMF) (experiments
  \#3--\#5).
\item To assess the computational demands of the proposed FPD-BTL
  algorithm in comparison to the competitive methods (BCM and DSMF).
\end{enumerate}

Section~\ref{ss:shapegraphs} explains the necessary background,
emphasizing the important distinction between the \emph{synthetic} and
\emph{analytic} model in these simulation studies. Then, model
mismatch and its types are specified (Sections~\ref{sss:noisemis}
and~\ref{sss:matrixmis}).  The specific LSU-UOS
systems~\eqref{e:LSU-dupdt}, \eqref{e:LSU-tupdt} used in our studies
are described in Section~\ref{ss:systems}.  The completely modelled
alternatives (BCM and DSMF) are reviewed in
Section~\ref{ss:methodcompar}, and the evaluation criteria are defined
in Section~\ref{ss:evalucrit}.  Then, the experimental results are
presented and discussed in Section~\ref{ss:results}, before overall
findings are collected and interpreted in Section~\ref{ss:discussion}.

\subsection{Synthetic vs. analytic models}                 \label{ss:shapegraphs}

In computer-based simulations---such as those which follow---we
explicitly distinguish between the \emph{synthetic} model, used for
data generation, and the \emph{analytic} model on which the derived
state estimation algorithm depends. The synthetic model can be
understood as an abstraction of a natural (physical) data-generating
process, while the analytic model is a~subjective
(i.e.\ epistemic~\cite{Jay:03})---and inevitably
\emph{approximate}---description of this process adopted by the
inference task (here, the LSU-UOS filters).

\noindent
\begin{figure}[h!]
  \centering
\begin{tabular}{ccccc}
\includegraphics[width=0.22\textwidth]{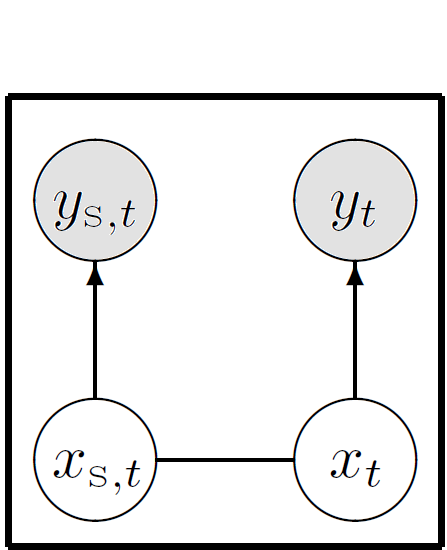}
&
\hspace{1.4em}
&
\includegraphics[width=0.22\textwidth]{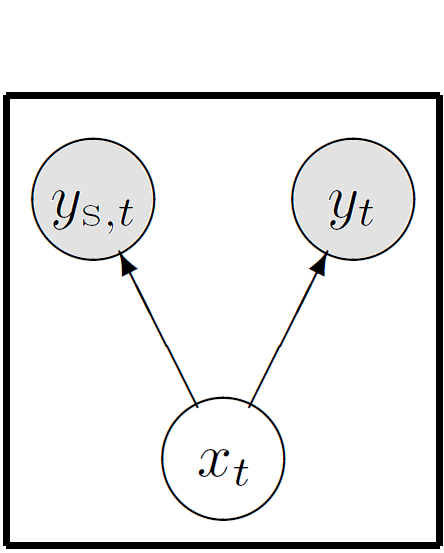}
&
\hspace{1.9em}
&
\includegraphics[width=0.31\textwidth]{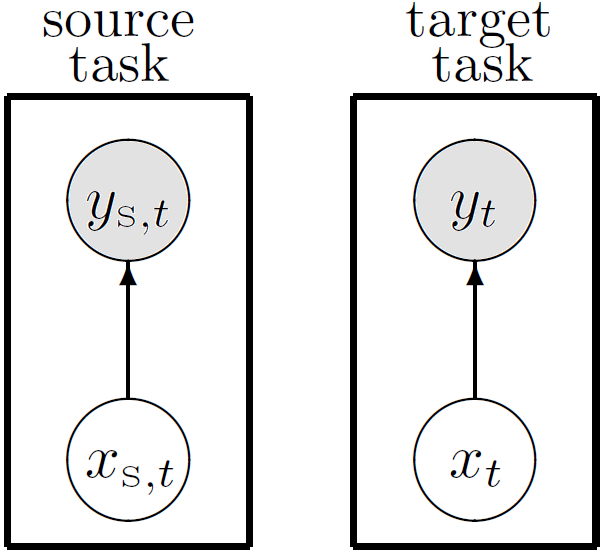}
\\
a) && b) && c)
\end{tabular}
\caption{Models (synthetic and/or analytic) for a pair of state-space
  observation processes, $y_{\sx,t}$ and $y_t$. a) V-shaped graph,
  single modeller; b) U-shaped graph, single modeller; c) multiple
  modellers. A~frame denotes a~modeller, that observes data (shaded
  nodes), and for which stochastic dependencies within the frame are
  known.  Conditional and joint probability models of these
  dependencies are represented by arrows and by lines (i.e.\ undirected
  and directed edges), respectively, as usual~\cite{KolFri:09}. The
  schemes show the marginal relationship between states and
  observations within one filtering step of dynamic modelling, but not
  the (temporal) dynamics themselves.  }

\label{f:shapegraphs}
\end{figure}

Figure~\ref{f:shapegraphs} shows three models for a pair of
state-space filters adopted in this paper, as either synthetic or
analytic models (or both).
If the V-shaped graph (Figure~\ref{f:shapegraphs}a) is used as the
synthetic model, state sequence $\left\{x_t\right\}_{t=1}^{\hi{t}}$ is
realized commonly for all the filters, via the Markov state
process~\eqref{eqn:SS-model-state-evo}, and then locally corrupted via
independent, additive, white UOS observation noise processes
$v_{k,t}$~\eqref{eqn:SS-model-observ}.
If the analytic model is also the V-shaped graph
(Figure~\ref{f:shapegraphs}a) with \emph{known} parameters, then we
refer to this as complete modelling, as adopted in BCM and DSMF.

The U-shaped graph (Figure~\ref{f:shapegraphs}b) is adopted as the
synthetic model in some of the experiments below. here, the target
state sequence, $\left\{x_t\right\}_{t=1}^{\hi{t}}$ and source state
sequence $\left\{x_{\sx,t}\right\}_{t=1}^{\hi{t}}$ are synthesized as
distinct---but mutually correlated---processes, with an appropriate
fully specified interaction model between $x_t$ and $x_{\sx,t}$.  The
U-shaped graph is not used as an analytic model in this paper.

As already explained in sections~\ref{sec:problem}
and~\ref{sec:main-res}, the multiple modeller approach
(Figure~\ref{f:shapegraphs}c) is adopted as the analytic model only in
our proposed FPD-BTL approach, expressing the fact that the target
elicits no model of the source process or of its relationship to
it. The source and target analytic models are therefore stochastically
independent, and can be interpreted as independent 2-node marginals of
a~(unspecified) 4-node complete model.  This arrangement respects the
key notion of \emph{local expertise}, i.e.\ the commonly encountered
situation in distributed inference where the source is a better local
analytic modeller (i.e.\ expert) of its local data than the remote
target modeller ever can be.

In the simulation studies most frequently encountered in the
literature, the synthetic and analytic models are implicitly assumed
to be identical. In the computer-based synthetic-data experiments
below, \emph{modelling mismatch} can be explored, and is, indeed, our
priority.  However, in real-data studies, the notion of a synthetic
model is inadmissible~\cite{Jay:03}. It follows that there is
therefore almost-sure mismatch between the (typically unknowable)
``truth theory'' of synthesis~\cite{Jay:03} and the analysis model
which prescribes the adopted algorithm. It is for this reason that a
study of analysis-synthesis modelling mismatch---as provided
below---is of key importance, particularly in assessment of the
robustness or fragility of the algorithm to such mismatches.

In the forthcoming simulations, modelling mismatch will be arranged at
the level of the state sequence(s), either via mismatches in the
process noise $w_t$~\eqref{eqn:SS-model-noises-unif}, or state matrix
mismatches, $A$~\eqref{eqn:SS-model-state-evo}. These are detailed in
the next two subsections.

\subsubsection{State noise mismatch}                         \label{sss:noisemis}

The target filter's state at time $t$ is synthesized according
to~\eqref{eqn:SS-model-state-evo} with a~uniform noise,
$w_t$~\eqref{eqn:SS-model-noises-unif}.
If the sequence $\left\{x_t\right\}_{t=1}^{\hi{t}}$ is common for both
the filters, data synthesis is described by the V-shaped graph
(Figure~\ref{f:shapegraphs}a), as already noted in
Section~\ref{ss:shapegraphs}. Synthesis via the U-shaped graph
(Figure~\ref{f:shapegraphs}b) realizes distinct state processes, $x_t$
and $x_{\sx,t}$, via the operating parameter, $\alpha\geq0$, which
controls the interaction (i.e.\ correlation) between them:
\begin{align}                                  \label{e:ugraph}
    x_{\sx,t} &= A x_{\sx,t-1} + B u_{\sx,t-1} + w_{\sx,t},
     &w_{\sx,t} &\sim \mathcal{U}(-\xnparv,\xnparv), \nonumber \\
    x_{t} &= x_{\sx,t} + e_{t}, & e_{t} &\sim
    \mathcal{U}(-\alpha\xnparv,\alpha\xnparv).
\end{align}
Here, $w_t$ and $e_t$ are mutually independent, white UOS state
processes in $\mathbb{R}^{\xsize}$. The source's analytic model is
also given by~\eqref{eqn:SS-model-state-evo} with perfectly matched
parameters.  However, we will assume that the target modeller is
\emph{unaware} of the mismatching noise process, $e_t$, and so the
target's analytic model of their local state process, $x_t$, is
also~\eqref{eqn:SS-model-state-evo}. This enforces a mismatch between
the target's synthetic model~\eqref{e:ugraph}, and its analysis model
\eqref{eqn:SS-model-state-evo}, via the state noise mismatch process,
$e_t$~\eqref{e:ugraph}.
Note that if $\alpha=0$, the source and target synthesized states are
identical (Figure~\ref{f:shapegraphs}a)
and matched to the source and target analysis models(s). In contrast,
if $\alpha>0$, then the marginal synthesis model (i.e.\ pdf) of $x_t$
is trapezoidal (being the convolution of two uniform pdfs) with
increased variance, while the target's mismatched local analytic state
model is uniform \eqref{e:LSU-tupdt}.

\subsubsection{State matrix mismatch in the analytic models}        \label{sss:matrixmis}

In this section, we distinguish between the state matrix
\eqref{eqn:SS-model-state-evo} in synthesis, $\asynt$, of the common
state process, $x_t$, (Figure~\ref{f:shapegraphs}a) ), and the state
matrix/matrices used in analysis, $\aanal$. Specifically, we set
$\aanal_\sx=\asynt$ (i.e.\ no synthesis-analysis mismatch in the
source), but $\aanal\neq\asynt$ (i.e.\ mismatch in the target).
There are several ways to achieve $\aanal\neq\asynt$:
\begin{enumerate}
\item Modification of the eigenvalues of invertible
  $\asynt=V\synt\Lambda\synt {V\synt}^{-1}$.  In the target analytic
  model, we modify the eigenvalues of $\aanal$ geometrically in one of
  two ways:
  \begin{enumerate}
  \item Radial shift: a selected eigenvalue $\lambda\synt_{i}$, $i \in
    \{1,\ldots,\xsize\}$ of $\asynt$ is multiplied by a~real scalar
    operating parameter $q>0$, i.e.\ $\lambda\anal_i\equiv
    q\lambda\synt_{i}$, while maintaining Hermitian symmetry.
  \item Rotation: here, $\lambda\synt_{i}$ is multiplied by a factor,
    $\mbox{e}^{j\varphi}$, where the angle of rotation, $\varphi$, is
    the operating parameter, i.e.\ $\lambda\anal_i\equiv
    \mbox{e}^{j\varphi}\lambda\synt_{i} $. Once again, Hermitian
    symmetry is maintained.
  \end{enumerate}
\item Multiplication of $\asynt$ by a scalar, $\sigma > 0$,
  i.e.\ $\aanal\equiv \sigma\asynt$. In this case, all eigenvalues of
  $\asynt$ experience the same radial shift.
\end{enumerate}
%

\subsection{The synthesis models}                          \label{ss:systems}

The following specific LSU systems~\eqref{eqn:SS-model-observ},
\eqref{eqn:SS-model-state-evo}, are simulated in the upcoming
experiments, i.e.\ they specify the synthesis model for both
$y_{\sx,t}$ and $y_t$ in the V-shaped graph
(Figure~\ref{f:shapegraphs}a) or in the U-shaped graph
(Figure~\ref{f:shapegraphs}b), as specified in \eqref{e:ugraph}. The
uncertainty parameters, $\ynparv$
and~$\xnparv$~\eqref{eqn:SS-model-noises-unif}, are specified in each
experiment.

\begin{itemize}
\item
A second-order system with two complex conjugate poles, described by
\eqref{eqn:SS-model-observ} and \eqref{eqn:SS-model-state-evo}, with
$\xsize=2$, $\usize=1$, $\ysize=1$, and
\begin{equation}                                    \label{e:simple_system}
  A =
  \left[
    \begin{array}{rr}
    0.8144 & -0.0905\\
    0.0905 &0.9953
    \end{array}
  \right]\!, \
  B =
  \left[
    \begin{array}{l}
    0.0905\\
    0.0047
    \end{array}
  \right]\!, \
  C =
  \left[
    \begin{array}{ll}
   0  &  1 \\
    \end{array}
  \right]\!.
\end{equation}

This system is studied in \cite{Fri:12}, being the discretization and
randomization of the continuous-time system, $\ddot{y}(\tau)+
2\dot{y}(\tau) + y(\tau) = u(\tau)$, with sampling period,
$T_0=0.1\,s$, and with added random processes, $v_t$ and $w_t$,
representing observational and modelling (i.e.\ state) uncertainties,
respectively.
\item
A third-order system with 3 distinct real poles, described by
\eqref{eqn:SS-model-observ} and \eqref{eqn:SS-model-state-evo} with
$\xsize=3$, $\usize=1$, $\ysize=2$, and
\begin{equation}                                   \label{e:another_system}
  \hspace{-0.3em}
  A =
  \left[
    \begin{array}{rrr}
      0.4  & -0.3 & 0.1 \\
      -0.4 & 0.4 & 0 \\
      0.3 & 0.2 & 0.1
    \end{array}
  \right]\!, \hspace{1em}
  B =
  \left[
    \begin{array}{l}
      0.1 \\ 0.6 \\ 0.3
    \end{array}
  \right]\!, \hspace{1em}
  C =
  \left[
    \begin{array}{lll}
      1 & 0 & 0.5 \\
      0 & 1 & 0.5
    \end{array}
  \right]\!.
\end{equation}
\end{itemize}

\subsection{Alternative multivariate inference algorithms} \label{ss:methodcompar}

The key distinguishing attribute of our FPD-BTL algorithm is its
multiple modeller approach with incomplete modelling of the
interaction between the tasks. Its defining characteristic---the
transfer of source sufficient statistics and not raw data---for
processing at the target, distinguishes it from methods that adopt a
complete model of the networked tasks, often involving joint
processing---at the target or other fusion centre---of the multiple
raw data channels. We will reserve the term transfer learning (TL) for
the former (FPD-BTL in the case of our FPD-optimal Bayesian TL
scheme), and refer to the latter as multivariate inference schemes.
We will compare FPD-BTL against two approaches to the latter:
(i)~Bayesian multivariate inference (Section~\ref{sss:bcm}) consistent
with a complete analysis model (i.e.\ V-shaped network graph in
Figure~\ref{f:shapegraphs}a); and (ii)~distributed set-membership
fusion (DSMF) (Section~\ref{sss:dsmf}), a~state-of-the-art,
non-probabilistic, fusion-based state estimation
algorithm~\cite{WanSheXiaZhu:19}.

\subsubsection{Bayesian complete modelling (BCM)}       \label{sss:bcm}

Here, it is assumed that the $n$ LSU filters, indexed by
$i=1,\ldots,n$, consist of $n$ conditionally independent observation
models with common Markov state evolution
model~\eqref{eqn:SS-model-state-evo} (i.e.\ the V-shaped graph as
analytic model, Figure~\ref{f:shapegraphs}a).  The $n$ observation
models are
\begin{equation}                                     \label{e:multobs}
  y_{i,t}= C_i x_t + v_{i,t}, \hspace{4em} i=1,\ldots,n,
\end{equation}
where the known $C_i\equiv C$ are possibly common, too, and the $n$
UOS white noise channels are mutually independent. The filters are
modelled by a~central modeller with knowledge in all the parameters
and data, i.e.\ the inference algorithm processes all $n$ data
channels, $y_{i,t}$.

The data update~\eqref{eqn:data-updt} uses the joint pdf of the $n$
data channels to process $y_{1t},\ldots,y_{nt}$ (of possibly different
dimensions):
\begin{equation}                                       \label{e:dupdc}
  f(x_{t}|d_1(t), \ldots, d_n(t))
  \propto
  f(x_{t}|d_1(t-1),\ldots,d_n(t-1))\ \prod\limits_{i=1}^n
  f(y_{i,t}|x_{t}),
\end{equation}
initialized (at $t=1$) with the prior,
$f(x_{1}|d_1(0),\ldots,d_n(0))\equiv f(x_1)$.

The time update~\eqref{eqn:time-updt} is
\begin{equation}                                        \label{e:tupdc}
f(x_{t+1}|d_1(t), \ldots, d_n(t)) =
\int\limits_{\mS{x}_{t}}\!\!
f(x_{t+1}|u_{t},x_{t})\,f(x_{t}|d_1(t),\ldots,d_n(t)) \diff x_{t}.
\end{equation}

The algorithmic solution of BCM, closed within the LSU-UOS
class---i.e.\ with $n+1$ projections into the UOS class, per step of
multivariate Bayesian filtering, as explained in
Section~\ref{sec:prelim}---is straightforward, because all the pdfs
have the same parametric form as in the
single-output ($n=1$) case. In~\eqref{e:dupdc}, the $n$ data updates
are computed by~\eqref{eqn:data-updt-unif2}
and~\eqref{eqn:data-updt-UOS-approx}---with arbitrary order---for each
$f(y_{i,t}|x_{t})$ in turn. Then, the LSU-closed approximate time
update is implemented via~\eqref{e:tupdc}, in the usual way
(Section~\ref{sec:time-updt-UOS}).

\subsubsection{Distributed set-membership fusion (DSMF)} \label{sss:dsmf}

The DSMF estimation task for a multi-sensor dynamic system with
unknown but bounded noises is presented in~\cite{WanSheXiaZhu:19}.
The state-space model~\eqref{eqn:SS-model-state-evo}
and~\eqref{e:multobs} is driven by (unmodelled) noises conditioned to
the following ellipsoidal sets:
\begin{align}            \label{eqn:SS-model-noises-ellips}
   \nonumber
   \mathbb{V}_{i,t} &= \{v_{i,t}: v'_{i,t}R^{-1}_{i,t} v_{i,t} \leq 1\},
   \hspace{2em}i=1,\ldots,n, \\
   \mathbb{W}_t &= \{w_t: w'_tQ^{-1}_t w_t \leq 1\}.
\end{align}
Here, matrices $R_{i,t}$ and $Q_t$ are symmetric and positive definite
of appropriate dimensions. The ellipsoids' centres are at the origin,
enforcing a prior assumption of zero (time-averaged) mean noises.

While, the method is non-probabilistic, it adopts complete modelling
via the V-shaped (analytic) graph (Figure~\ref{f:shapegraphs}a), with
the edges defined only via geometric relationships. The resulting
recursive algorithm consists of a deterministic time update of the
combined state estimate from the previous step, then a~data update for
each filter, and finally a fusion step, when individual state
estimates are combined, using also the time-updated state.
Optimal fusion weights are calculated by solving a~convex optimization
problem in each step of the algorithm.

\subsection{Evaluation criteria}                       \label{ss:evalucrit}

For competitive quantitative evaluation of FPD-BTL, BCM and DSMF, we
define the following performance quantities:
\begin{itemize}
\item Total norm squared-error (TNSE) of the state estimate:
  \begin{equation}                              \label{e:deftnse}
  \mbox{TNSE}=\sum\limits_{t=\lo{t}}^{\hi{t}}\|\hat{x}_t-x_t \|_2^2,
  \end{equation}
$\lo{t}=1$ for FPD-BTL and BCM, but is set higher---typically
  $t/2$---in DSMF, as explained later.  $\|\cdot\|_2$ is the Euclidean
  norm. In the FPD-BTL and BCM, $\hat{x_t}$ is mean value of the UOS
  filtering pdf~\eqref{eqn:data-updt-UOS-approx}, either
  via~\eqref{e:dupd-optim} (i.e.\ line~2 of Algorithm~\ref{alg:alg}),
  or via~\eqref{e:dupdc}.  In DSMF, $\hat{x_t}$ is centre of the
  state's estimated ellipsoid.
\item Average posterior volume (AV). Define sequential
  $\mS{x}_t\equiv\supp\left(f(x_t|d(t),f_\sx)\right)$ in the case of
  FPD-BTL~\eqref{e:dupd-optim}, and
  $\mS{x}_t\equiv\supp\left(f(x_t|d(t),d_\sx(t))\right)$ in the case
  of BCM~\eqref{e:dupdc} (both orthotopic), $\mS{x}_t\equiv\left\{x_t:
  x_t'P_t^{-1}x_t\leq 1 \right\}$ (ellipsoidal) in the case of DSMF,
  where $P_t$ is a~sequentially computed matrix defining the bounding
  ellipsoid of~$\mS{x}_t$. Let the Lebesgue measure (hypervolume) of
  $\mS{x}_t$ be $\mbox{V}_t=\mu\left(\mS{x}_t\right)$ in each
  case. Then
  \begin{equation}                               \label{e:defv}
    \mbox{AV} = \frac{1}{\hi{t}-\lo{t}+1}\ \sum\limits_{t=\lo{t}}^{\hi{t}}
    {\mbox{V}_t}.
  \end{equation}
\item Average volume ratio (AVR). Let $\mS{x}_t$ be defined as for the
  AV above, respectively. Let $\mbox{VI}_t$ denote the Lebesgue
  measure of $\mS{x}_t$ in the case of the isolated target
  (Figure~\ref{f:shapegraphs}c) (or without fusion, in the case of
  DSMF) and $\mbox{VT}_t$ the same respective quantity with
  transfer/fusion. Then
  \begin{equation}                               \label{e:defvr}
    \mbox{AVR} = \frac{1}{\hi{t}-\lo{t}+1}\ \sum\limits_{t=\lo{t}}^{\hi{t}}
    \frac{\mbox{VT}_t}{\mbox{VI}_t}.
  \end{equation}
It is an average ratio of concentration of the state inference
under transfer fusion

\item Containment probability ($p\subs{c}$): this is the probability
  that the simulated---i.e.\ ``true''---state $x_t$ is contained in
  the respective inferred $\mS{x}_t$, as defined for AV above. This probability is
  calculated as a relative frequency of containment occurrence:
  \begin{equation}                             \label{e:defcpr}
  p\subs{c} = \frac{1}{\hi{t}-\lo{t}+1}\ \sum\limits_{t=\lo{t}}^{\hi{t}}
  \chi\left(x_t\in\mS{x}_t \right).
  \end{equation}
\item Computation time for comparison of FPD-BTL with DSFM by multiple
  executions with various noise parameters, run on PC Matlab.
\end{itemize}

In the experiments below, the vector noise parameters,
$\ynparv\in\mS{r}^\ysize$ and
$\xnparv\in\mS{r}^\xsize$~\eqref{eqn:SS-model-noises-unif}, are
isotopic, i.e.
\begin{equation}                                     \label{e:defnpars}
  \ynparv=\ynpars\,\unitvect{\ysize}, \hspace{3em}
  \xnparv=\xnpars\,\unitvect{\xsize},
\end{equation}
where $\xnpars$ and $\ynpars$ are positive scalars and $\unitvect{k}$
is the $k$-dimensional unit vector.

The figures below graph the dependence of TNSE~\eqref{e:deftnse},
average volume (AV)~\eqref{e:defv} and average volume ratio
(AVR)~\eqref{e:defvr} against the source-to-target ratio of the
observation noise variances, i.e.\ $\ynpars_\sx/\ynpars$.  For all
these performance quantities, smaller is better. Conversely, the
containment probability, $p\subs{c}$~\eqref{e:defcpr}, where
evaluable, is a~bigger is better quantity. Because of their wide
numerical ranges, the quantities~\eqref{e:deftnse}, \eqref{e:defv}
and~\eqref{e:defvr} are plotted logarithmically.  They are also
evaluated for the (isolated) $y_t$ channel in each case
(Figure~\ref{f:shapegraphs}) providing a benchmark (i.e.\ datum)
against which to assess the impact of the various $y_\sx$-processing
algorithms. FPD-BTL and BCM will be identical in this isolated $y_t$
case.
In these non-$d_{\sx,t}$ cases, performance is of course invariant
with $\ynpars_\sx/\ynpars$.

For each setting of $\ynpars_\sx/\ynpars$, the performance graphs are
obtained as the average over MC$\equiv$50 or 500 Monte Carlo runs,
holding all the synthetic model parameters ($A$, $B$, $C$, $\xnpars$,
$\ynpars$) constant~\eqref{eqn:SS-model-observ},
\eqref{eqn:SS-model-state-evo}, \eqref{e:defnpars}.

In the case of $n=2$ data channels, the operating parameter,
$\ynpars_\sx/\ynpars=1$, constitutes a~\emph{threshold}. If
$\ynpars_\sx/\ynpars<1$ (which we call the above-threshold region),
$y_{s,t}$ reduces uncertainty in $x_t$ in our FPD-BTL algorithm, as we
will see.  We call this \emph{positive transfer} (i.e.\ improved
performance relative to the isolated target task). Conversely, if
$\ynpars_\sx/\ynpars>1$ (below threshold), $y_{s,t}$ may undermine the
estimation of $x_t$, depending on the algorithm.  To be \emph{robust},
state estimation performance should revert to estimation conditioned
only on $y_t$ (i.e.\ the algorithm should be capable of rejecting
knowledge from $y_\sx$). A key priority of our studies below will be
to assess how robust our FPD-BTL algorithm is in comparison with the
alternatives.

\subsection{Experimental results}              \label{ss:results}

The modelling choices of the experiments are summarized in
Table~\ref{t:summaryexp}. Each row refers to a particular experiment
(\#1--\#5) to follow. Columns record which synthetic model---a or b in
Figure~\ref{f:shapegraphs}---is used to simulate the data, and
which---a or c---is adopted as the analytic model in the respective
algorithm (FPD-BTL, BCM and DSMF). The last column indicates whether
the state synthesis and (target) analysis matrices, $\asynt$ and
$\aanal$ respectively, are set equal or not (see
Section~\ref{sss:matrixmis}).

\begin{table}[h!]
  \centering
\begin{tabular}{lccccc} 
  \hline
             &           &    BTL       &    BCM    &   DSMF &
             \\[-0.8em]
  experiment & synthetic &  analytic & analytic & analytic &
  $\asynt\equiv \aanal$
  \\
  \hline
\#1 (Sec. \ref{sss:btlmany})     & a & c & -- & -- & \checkmark \\ \hline
\#2 (Sec. \ref{sss:btlcmdsmf})   & a & c & a & a & \checkmark \\ \hline
\#3 (Sec. \ref{sss:modmismatch}) & a & c & a & a &  $\times$ \\ \hline
\#4 (Sec. \ref{sss:modmismatch}) & a & c & a & a &  $\times$ \\   \hline
\#5 (Sec. \ref{sss:modmismatch}) & b & c & a & a &  \checkmark \\   \hline
\end{tabular}
\caption{The synthetic and analytic models adopted in each
  experiment. Symbols a, b, c refer to the graphical models in
  Figure~\ref{f:shapegraphs}.}
\label{t:summaryexp}
\end{table}

\subsubsection{FPD-BTL with multiple source tasks}
                                                 \label{sss:btlmany}
Experiment \#1:
Here, the performance of an isolated target LSU-UOS filter is compared
to FPD-BTL with multiple source filters.  State synthesis
model~\eqref{e:simple_system} is adopted with state sequence, $x_t$,
noisily observed in each of $n$ scalar i.i.d. data channels, $n=2, 11,
101, 1\,001$ (i.e.\ for $n_\sx=n-1=1, 10, 100, 1\,000$ source tasks,
respectively, plus 1 target task). In all settings, $\xnpars=10^{-5}$,
$\ynpars=10^{-3}$~\eqref{e:defnpars}, $\lo{t}=1$, $\hi{t}=50$,
MC\,=\,500~runs. Performances are evaluated via TNSE~\eqref{e:deftnse}
and AVR~\eqref{e:defvr}, in comparison to the isolated target
filter. We observe the following:
\begin{figure}[htb]
    \centering
    {~\hfill a) \hspace{0.49\textwidth} b) \hfill ~} \\[-0.2em]
   {\includegraphics[width=0.49\textwidth]{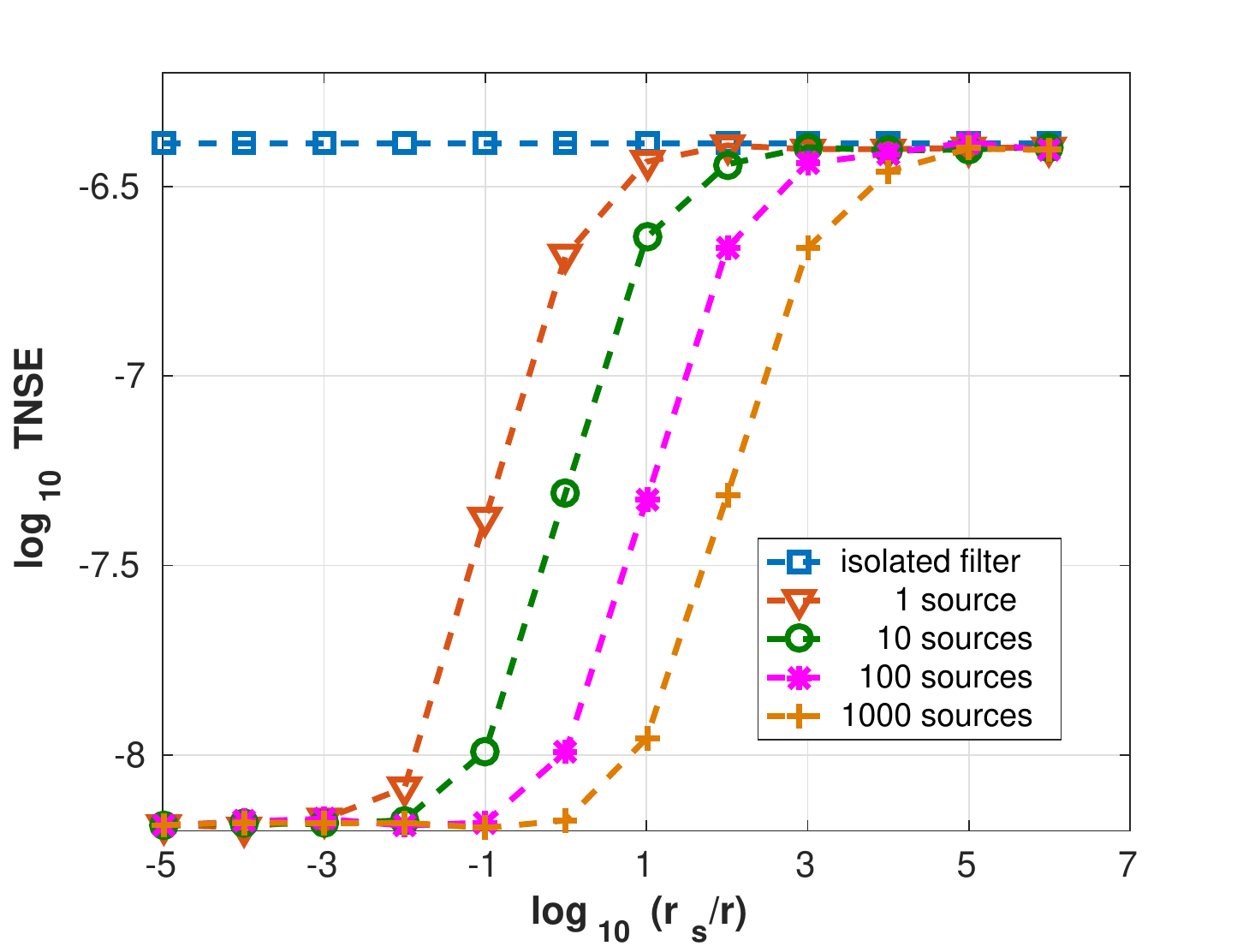}}
   \hfill
   {\includegraphics[width=0.49\textwidth]{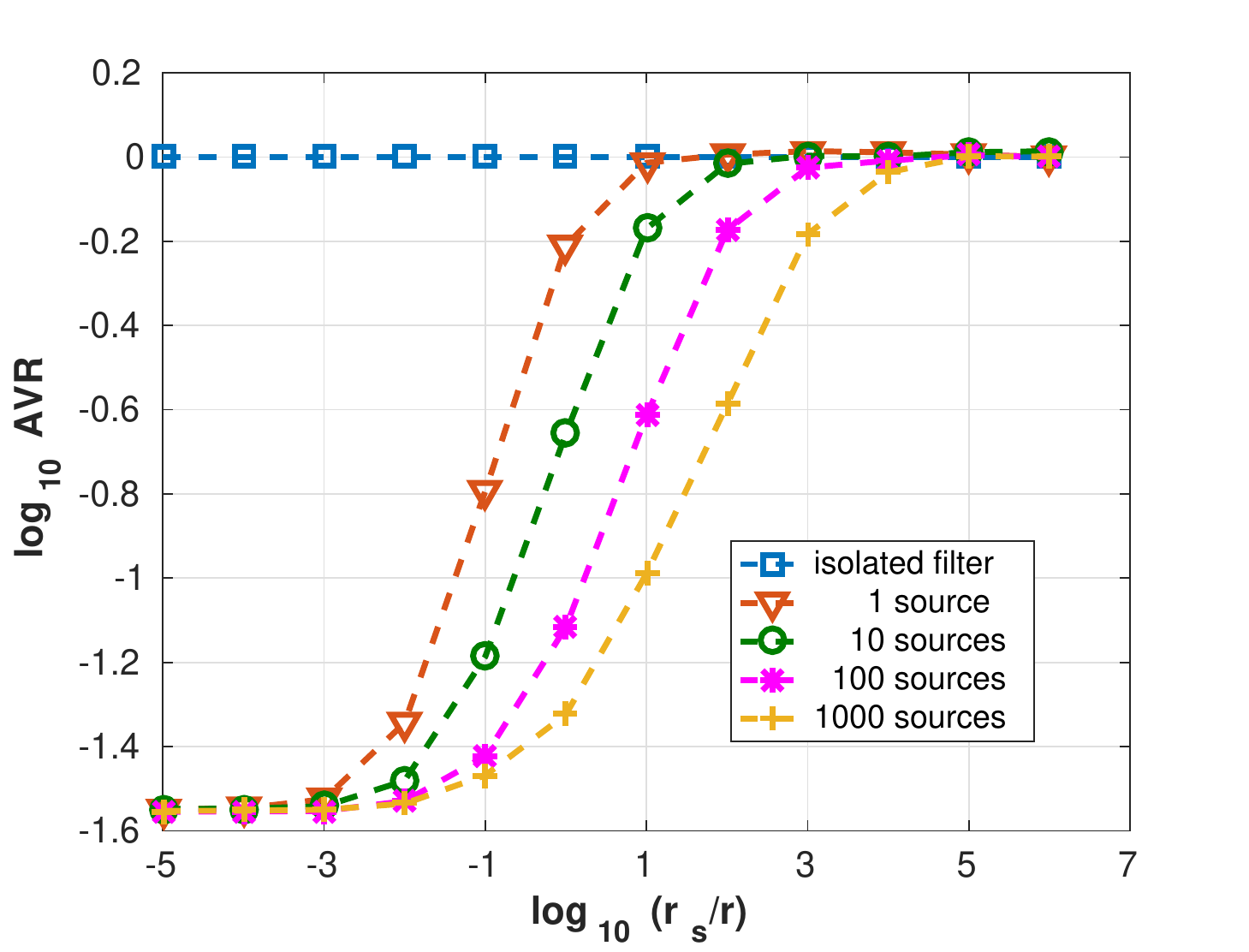}}\\
  \caption{Experiment \#1: Comparison of isolated target filtering
    task and FPD-BTL with multiple source filtering
    tasks. a)\,TNSE~\eqref{e:deftnse} and b)\,AVR~\eqref{e:defvr},
    system~\eqref{e:simple_system},
various numbers of source tasks.
    }
  \label{f:tnse_vol_simple_system}
\end{figure}
\begin{enumerate}[label=(\roman*)]
\item FPD-BTL exhibits \emph{positive transfer} above threshold; i.e.
  the transferred source knowledge improves the performance of the
  target.
\item FPD-BTL is \emph{robust}; i.e.\ the target becomes isolated from
  the sources (transfer is quenched) in cases of poor-quality source
  knowledge (below threshold, where the source state-predictive
  variance is high relative to that of the isolated target). In this
  way, \emph{negative transfer}---a~hazard of transfer learning
  algorithms~\cite{FolQui:18}, \cite{JirPavQui:20}---is eliminated.
\item
The $\ynpars_\sx/\ynpars$ threshold increases monotonically with the
number of sources.  For $n_\sx\!\!=\!\!1$, the threshold is
$\ynpars_\sx\!\approx\!10\,\ynpars$.
However, poor-quality sources continue to deliver positive transfer to
the target up to a threshold, $\ynpars_\sx\!=\!10n_\sx\, \ynpars$, in
the multi-source case. The threshold is found to be independent
of~$\hi{t}$.
\item
All performances saturate in very positive transfer regimes (i.e.
where $\ynpars_\sx\ll\ynpars$).
The bound is determined by $\xnpars$ and~$\ynpars$~\eqref{e:defnpars},
which are invariant.
\end{enumerate}

The containment probability, $p\subs{c}$~\eqref{e:defcpr}, is
identically 1 for all the runs, i.e.\ all the ``true'' (simulated)
states are contained in the posterior LSU-UOS support.

\subsubsection{Comparison of FPD-BTL, BCM and DSMF for $n=2$ data channels}
                                                    \label{sss:btlcmdsmf}

Experiment \#2:
In all forthcoming experiments, $n=2$, i.e.\ 2 data channels, with the
single LSU-UOS source task ($n_s=1$) processing $d_{\sx,t}$ and the
single target task processing $d_t$ in FPD-BTL. The fully modelled
alternatives (BCM and DSMF) process $d_{\sx,t}$ and $d_t$ together. In
the current experiment (\#2), common state process, $x_t\in\mS{r}^2$,
is synthesized via system~\eqref{e:simple_system} with
$\xnpars=10^{-5}$, and
$\ynpars=10^{-3}$ in the $n=2$ conditionally iid data channels,
$\lo{t}=2\,000$, $\hi{t}=4\,000$ and MC\,=\,50~runs.

\begin{figure}[htb]
    \centering
    {~\hfill a) \hspace{0.49\textwidth} b) \hfill ~} \\[-0.2em]
{\includegraphics[width=0.49\textwidth]{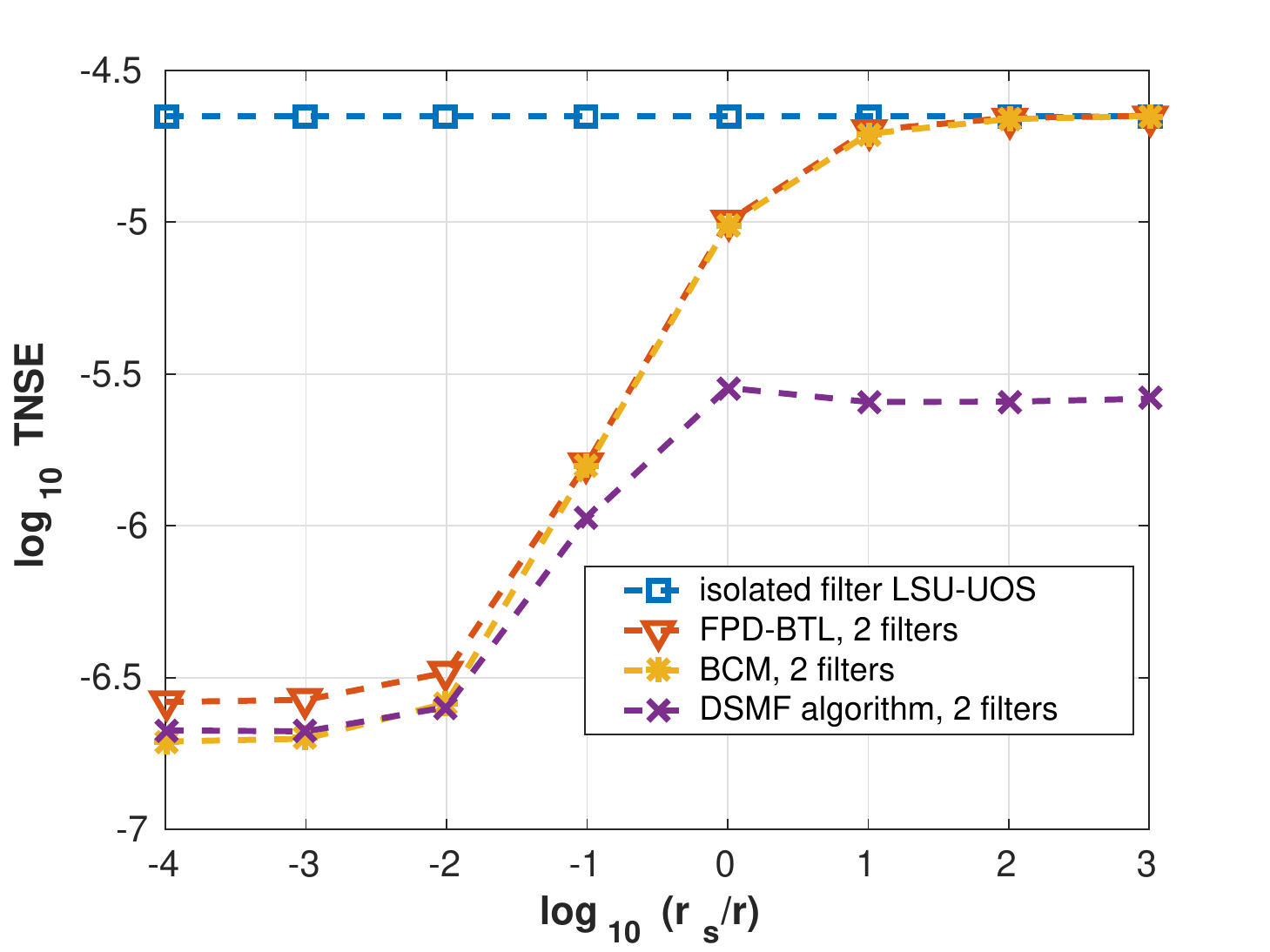}}
\hfill
{\includegraphics[width=0.49\textwidth]{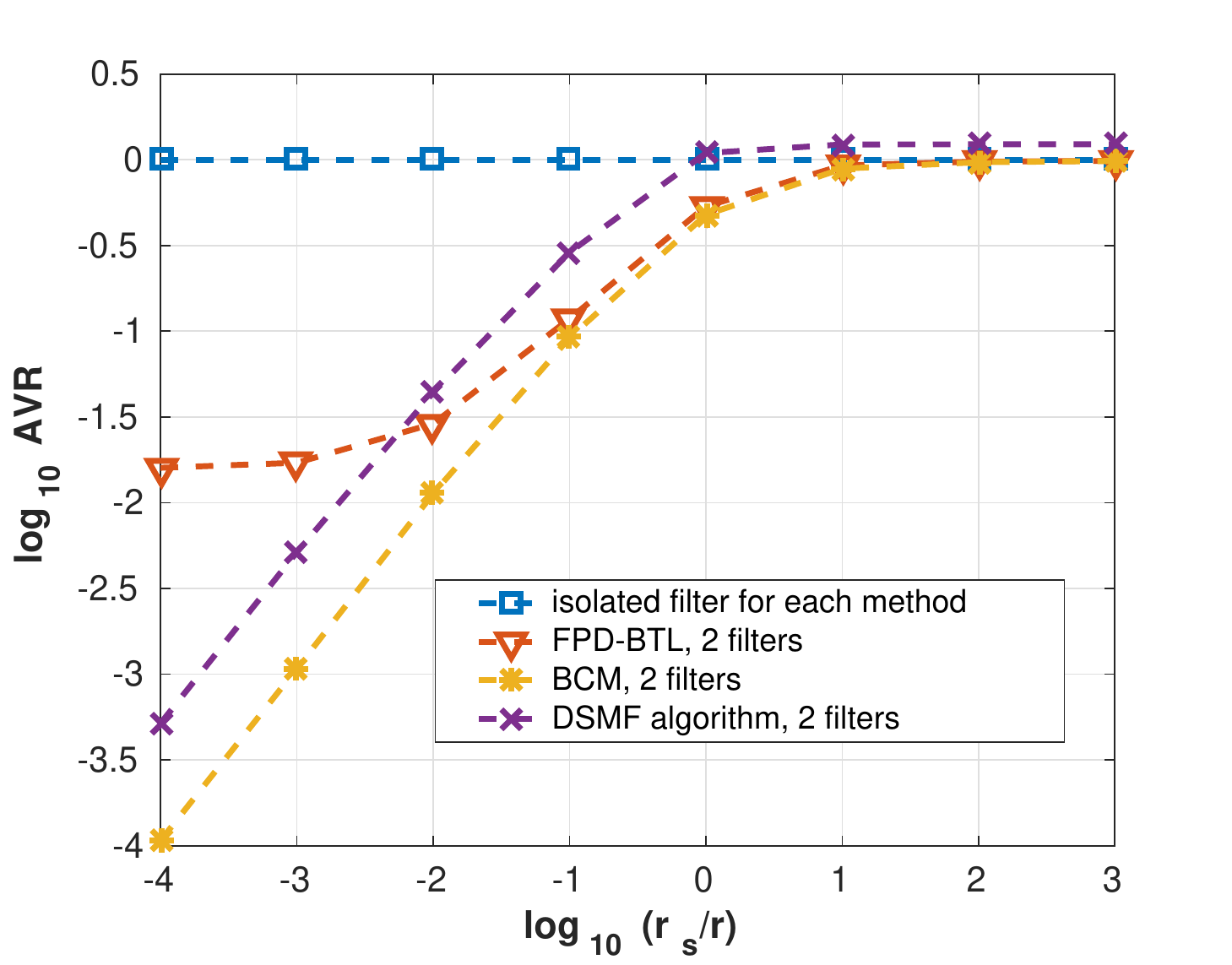}}
 \\
{c)}\\[-0.2em]
 {\includegraphics[width=0.49\textwidth]{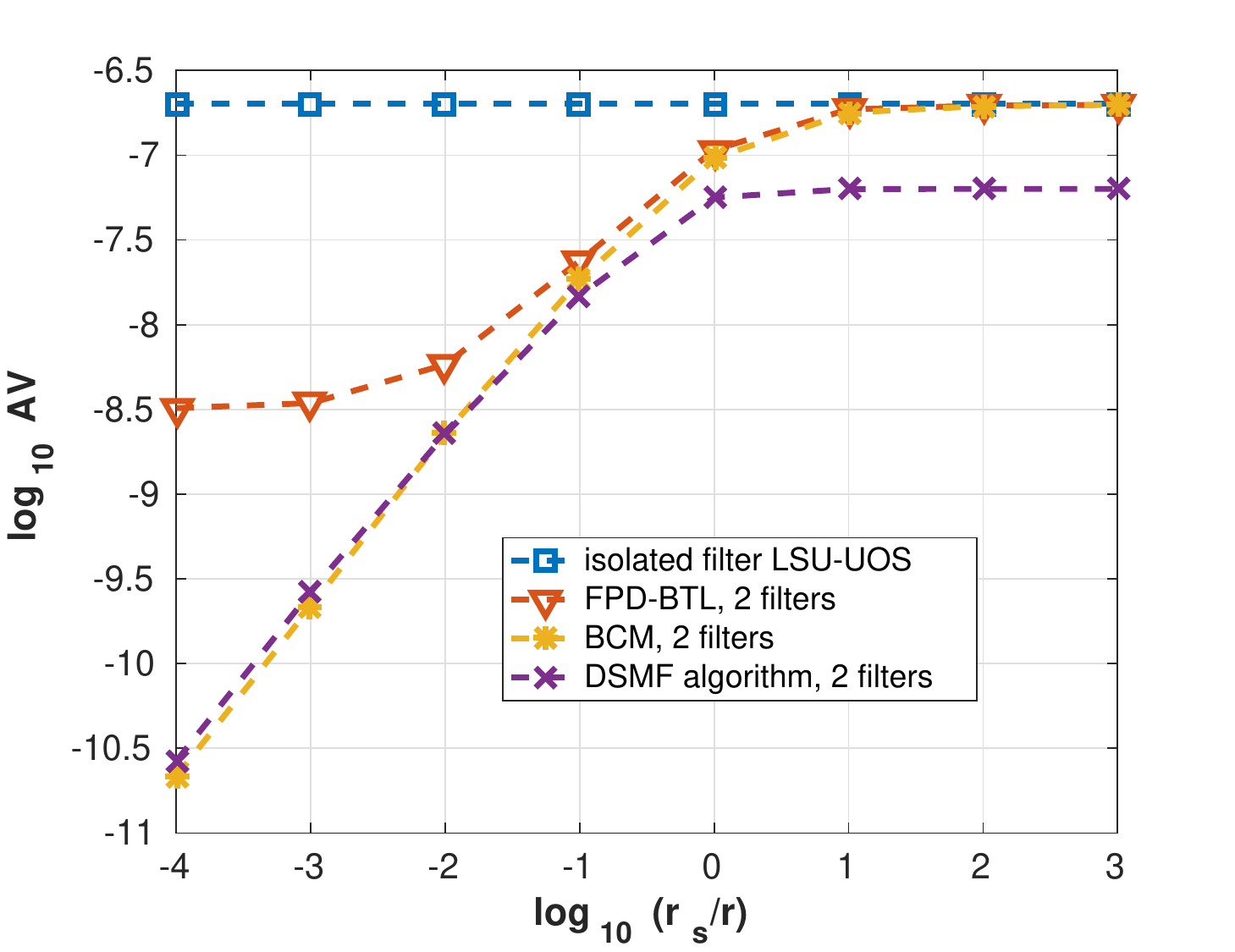}}
  \caption{Experiment \#2: Filtering performance of FPD-BTL compared
    to the isolated LSU-UOS task~\eqref{e:simple_system}, and to
    complete modelling alternatives (BCM and DSMF): a)
    TNSE~\eqref{e:deftnse}, b) AVR~\eqref{e:defvr} and c)
    AV~\eqref{e:defv}.  }
  \label{f:volr_simple_system_compl}
\end{figure}

The performance of the isolated target task (processing $d_t$ alone)
is compared to our FPD-optimal BTL scheme and to the joint processing
alternatives under three performance criteria
(Figure~\ref{f:volr_simple_system_compl}).  All the characteristics
are decreasing as $\ynpars_\sx$ decreases. We note the following:
\begin{enumerate}[label=(\roman*)]
\item The threshold for FPD-BTL and BCM is
  $\log_{10}\ynpars_\sx/\ynpars=1$, in agreement with
  experiment~\#1. The threshold for DSMF is
  $\log_{10}\ynpars_\sx/\ynpars=0$.
\item Below threshold ($\log_{10}\ynpars_\sx/\ynpars>1$), the absolute
  performance quantities (TNSE~\eqref{e:deftnse},
  Figure~\ref{f:volr_simple_system_compl}a, and AV~\eqref{e:defv},
  Figure~\ref{f:volr_simple_system_compl}c) coincide for the Bayesian
  LSU-UOS methods (FPD-BTL and BCM); the DSMF shows better
  performance.
\item Above threshold ($\log_{10}\ynpars_\sx/\ynpars<1$), the
  characteristics mentioned above coincide for complete modelling
  methods (BCM and DSMF), outperforming FPD-BTL.
\item Since the analytic and synthetic models match in the complete
  modelling approaches (BCM and DSMF), they provide better estimation
  than FPD-BTL, and DSMF dominates over BCM due to its tighter
  geometric approximation, although with a lower threshold.
\end{enumerate}

Finally, we comment on the containment probability, $p\subs{c}$ (not
shown in Figure~\ref{f:volr_simple_system_compl}), i.e.\ the
probability that the ``true'' (simulated) state is contained in the
inferred $\mS{x}_t$, respectively (see the second bullet of
Section~\ref{ss:evalucrit}). For the Bayesian LSU-UOS approaches
\emph{and if} the analytic and synthetic models match,
$p\subs{c}=1$. For DSMF, $p\subs{c}<1$ in many operating settings,
depending on the system, noise, length of the burn-in period etc.,
however, very close to 1. This suggests that the containing ellipsoid
in DSMF is estimated too small or mis-centered, missing the true state
occasionally.

Experiment \#2 was repeated with the system~\eqref{e:another_system},
providing very similar findings. However, above threshold, $p\subs{c}$
for DSMF decreases to about 0.91.

\subsubsection{Modelling mismatch}                     \label{sss:modmismatch}

In the remaining three experiments (\#3--\#5, see
Table~\ref{t:summaryexp}), our aim is to explore the freedom that
derives from the fact that FPD-BTL is a multiple modeller framework,
allowing the source modeller to adopt an analytic model different from
the target's (in all these experiments, we focus on $n_\sx=1$ source
modeller).

In particular, we explore the additivity that arises from the
case---often encountered in practical distributed inference and
multisensor settings---where the source is an \emph{expert}, meaning
that the analytic model for its (local) source data, $d_{\sx,t}$,
matches the synthetic model for $d_{\sx,t}$ better than the
analysis-synthesis arrangement for $d_t$ in the target task
(Figure~\ref{f:shapegraphs}c). Our experiments will focus on the case
where the target's analysis model for $d_t$ is mismatched with respect
to its synthesis model, while no such mismatch exists for the (expert)
source's data, $d_{\sx,t}$. The hypothesis we seek to test is that
positive transfer from the source's matched state predictor,
$f(x_{\sx,t}|d_\sx(t-1))$~\eqref{eqn:time-updt-UOS-approx}, can
improve the target's filtering performance, relative to complete
modelling approaches---BCM and DSMF
(Section~\ref{ss:methodcompar})---where mismatching occurs in
\emph{both} channels, $d_{\sx,t}$ and $d_{t}$.

\paragraph{Experiment \#3: Rotation of a pair of eigenvalues of $\asynt$}    \label{p:eigrot}

In the system~\eqref{e:simple_system}, the state transition matrix in
the synthesis model, $\asynt$, has a~pair of complex conjugate
eigenvalues $0.9049\pm0.003i$.  These eigenvalues are rotated
(mismatched) by angle $\pm \varphi$ (radians, preserving conjugacy),
in the target's analysis model (FPD-BTL), and in the completely
modelled approaches (BCM, DSMF). The interval, $\varphi\in(-0.007,
0.067)$, represents the range (found experimentally) when
$\mS{x}^\cap_t\neq\emptyset$ in Theorem~\ref{t:t1a}, and the
estimation is numerically stable, i.e.\ the intersection
in~\eqref{eqn:data-updt-unif2} is nonempty.
\begin{figure}[h!]
    \centering
    {~\hfill a) \hspace{0.49\textwidth} b) \hfill ~} \\[-0.2em]
{\includegraphics[width=0.49\textwidth]{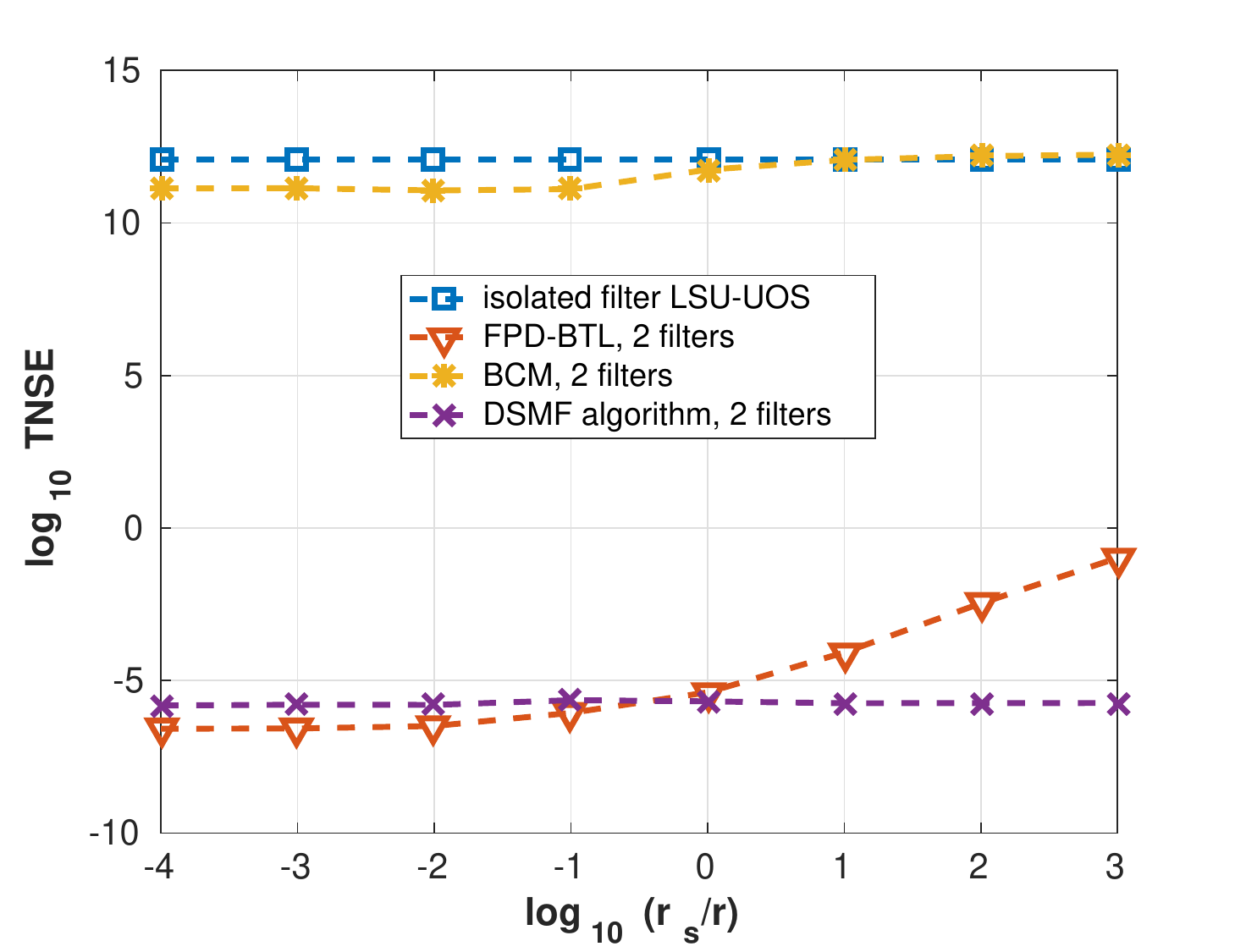}}
\hfill
{\includegraphics[width=0.49\textwidth]{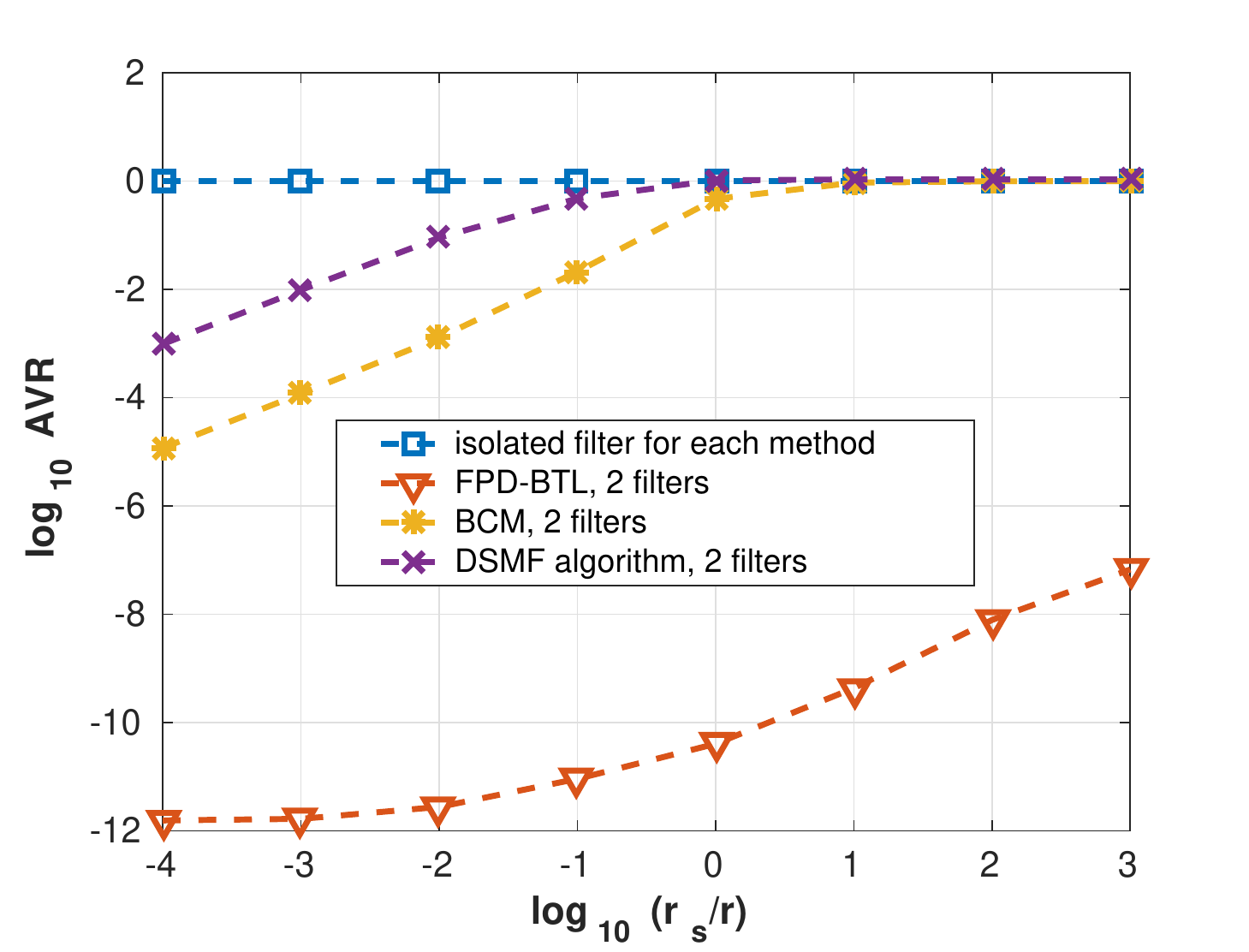}}
 \\
{c)}\\[-0.2em]
{\includegraphics[width=0.49\textwidth]{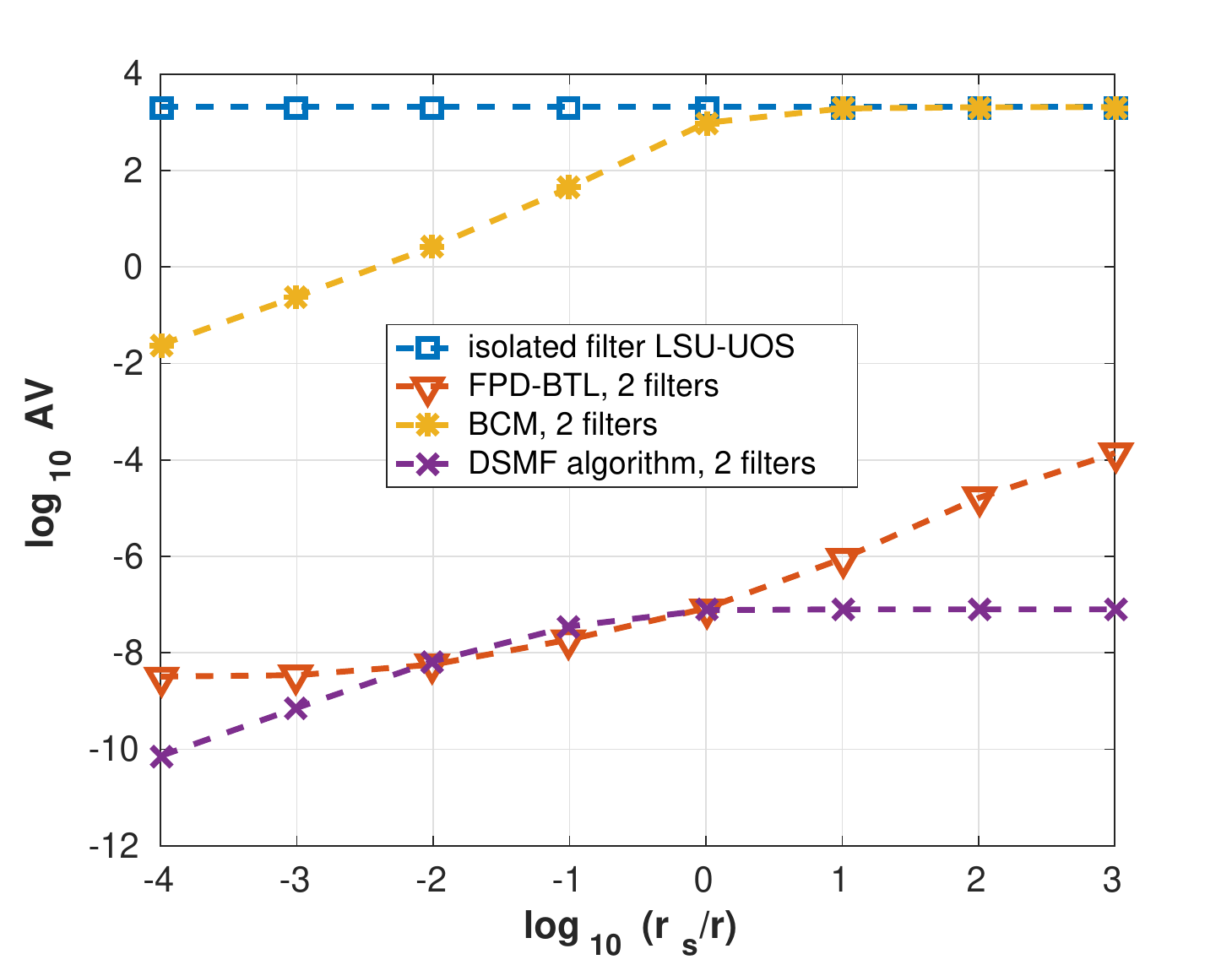}}
\caption{Experiment \#3: a) TNSE~\eqref{e:deftnse}, b)
  AVR~\eqref{e:defvr} and c) AV~\eqref{e:defv}, rotational
  mismatch---$\varphi=0.067$---of the eigenvalues of
  $A$~\eqref{e:simple_system} in the target's analysis model
  (FPD-BTL), and in the complete analysis models (BCM and DSMF,
  V-shaped complete model (Figure~\ref{f:shapegraphs}a)).}
  \label{f:mismatch_rot1}
\end{figure}
To repeat: the data are synthesized with $\varphi=0$, i.e.\ with the
$\asynt=A$ in~\eqref{e:simple_system}, according to the V-shaped graph
(Figure~\ref{f:shapegraphs}a). The same matrix, $\aanal=A$, is used
for the source analytic model. For the target analytic model, $A$ is
perturbed with $\varphi=0.067$. This bounding value of $\varphi$ in
the interval above is adopted in order to illustrate relative
algorithmic performance under significant analysis-synthesis mismatch.
Also, $\ynpars=10^{-3}$, $\xnpars=10^{-5}$~\eqref{e:defnpars},
$\lo{t}=2\,000$, $\hi{t}=4\,000$, MC\,=\,50~runs
(Figure~\ref{f:mismatch_rot1}).
We note the following:
\begin{enumerate}[label=(\roman*)]
\item The TNSEs for the (mismatched) isolated target filter and BCM
  are very high (Figure~\ref{f:mismatch_rot1}a), whereas the TNSE for
  FPD-BTL preserves the same positive transfer characteristic as was
  achieved with the matched target case (i.e.\ $\varphi=0$ ,see
  Figure~\ref{f:volr_simple_system_compl}a).
\item AV for FPD-BTL (Figure~\ref{f:mismatch_rot1}c) also preserves
  the same positive transfer characteristic as for $\varphi=0$,
  outperforming BCM.
\item DSMF is robust to this kind of modelling mismatch: its AV
  achieves values close to the matched ($\varphi=0$) case
  (Figure~\ref{f:volr_simple_system_compl}a). The TNSE of DSMF
  deteriorates insignificantly relative to the matched case, and is
  almost constant with $\ynpars_\sx/\ynpars$. In this experiment, DSMF
  outperforms the other methods.
\item The $\ynpars_\sx/\ynpars$ thresholds for BCM and DFSM are
  preserved.
\item FPD-BTL and DSMF are robust to $\ynpars_\sx\gg\ynpars$, in the
  sense that they achieve very low TNSE in this regime, both in
  absolute terms \emph{and} relative to the mismatched isolated
  target.
\end{enumerate}

Note that when $\varphi$ is varied from 0 (matching) to 0.067 (maximal
computationally stable mismatch), the performance measures (TNSE, AVR,
AV) increasingly deviate from the results with $\varphi=0$ (not
illustrated here).
The same is true as $\varphi$ decreases from 0 to $-0.007$, with the
results $\varphi=-0.007$ being similar to those shown for the other
extremum ($\varphi=0.067$) in Figure~\ref{f:mismatch_rot1}.

\paragraph{Experiment \#4: Dilation of eigenvalues of $\asynt$}

In this experiment, once again, the V-shaped model (see
Figure~\ref{f:shapegraphs} and Table~\ref{t:summaryexp}) is adopted in
the synthesis of the $\xsize=3$-dimensional state process, $x_t$, and
of the two $\ysize=2$-dimensional bivariate data channels, $d_{\sx,t}$
and $d_t$, for the system in~\eqref{e:another_system}, with
$\ynpars=10^{-3}$, $\xnpars=10^{-5}$~\eqref{e:defnpars}, $\lo{t}=200$,
$\hi{t}=400$, MC\,=\,500.
This system has 3 distinct real eigenvalues, which suffer a~common
dilation (i.e.\ scaling) mismatch in the target's analysis model,
$\aanal$; i.e.\ the target adopts the state transition matrix,
$\aanal=\sigma\asynt=\sigma A$~\eqref{e:another_system}, with
$\sigma=1.4$ (Section~\ref{sss:matrixmis}).
%
\begin{figure}[h!]
 \centering
    {~\hfill a) \hspace{0.49\textwidth} b) \hfill ~} \\[-0.2em]
{\includegraphics[width=0.49\textwidth]{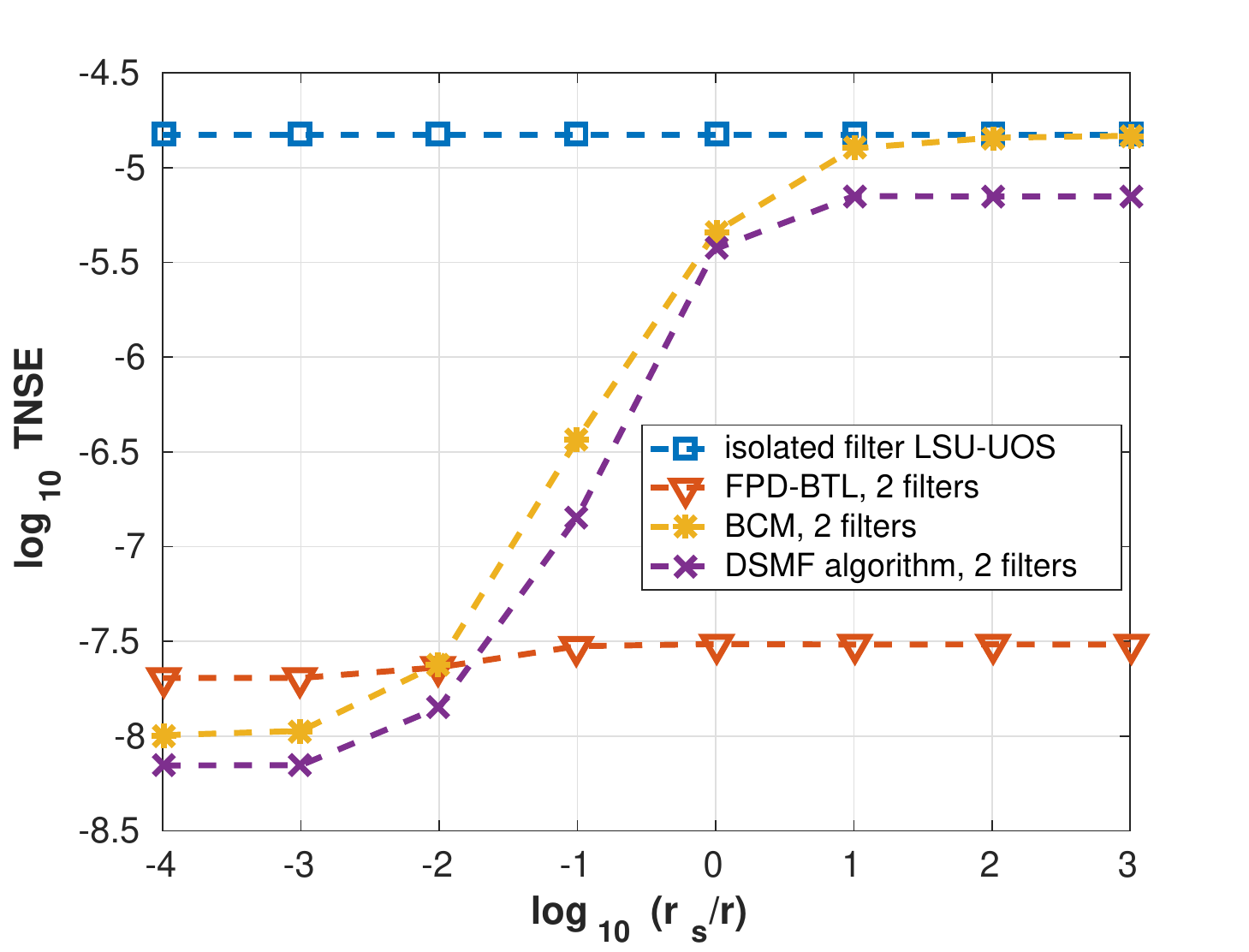}}
\hfill
{\includegraphics[width=0.49\textwidth]{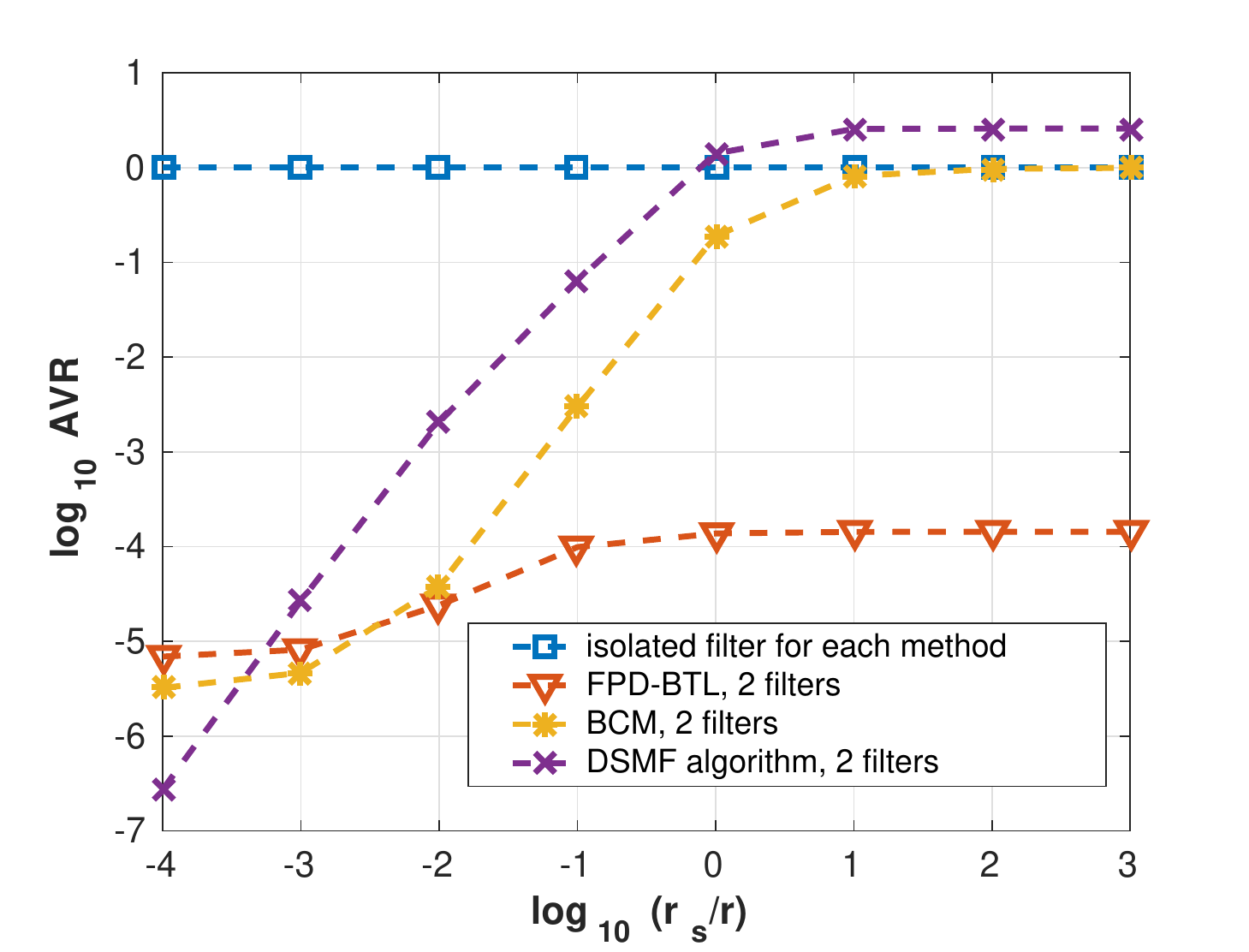}}\\
    {~\hfill c) \hspace{0.49\textwidth} d) \hfill ~} \\[-0.2em]
{\includegraphics[width=0.49\textwidth]{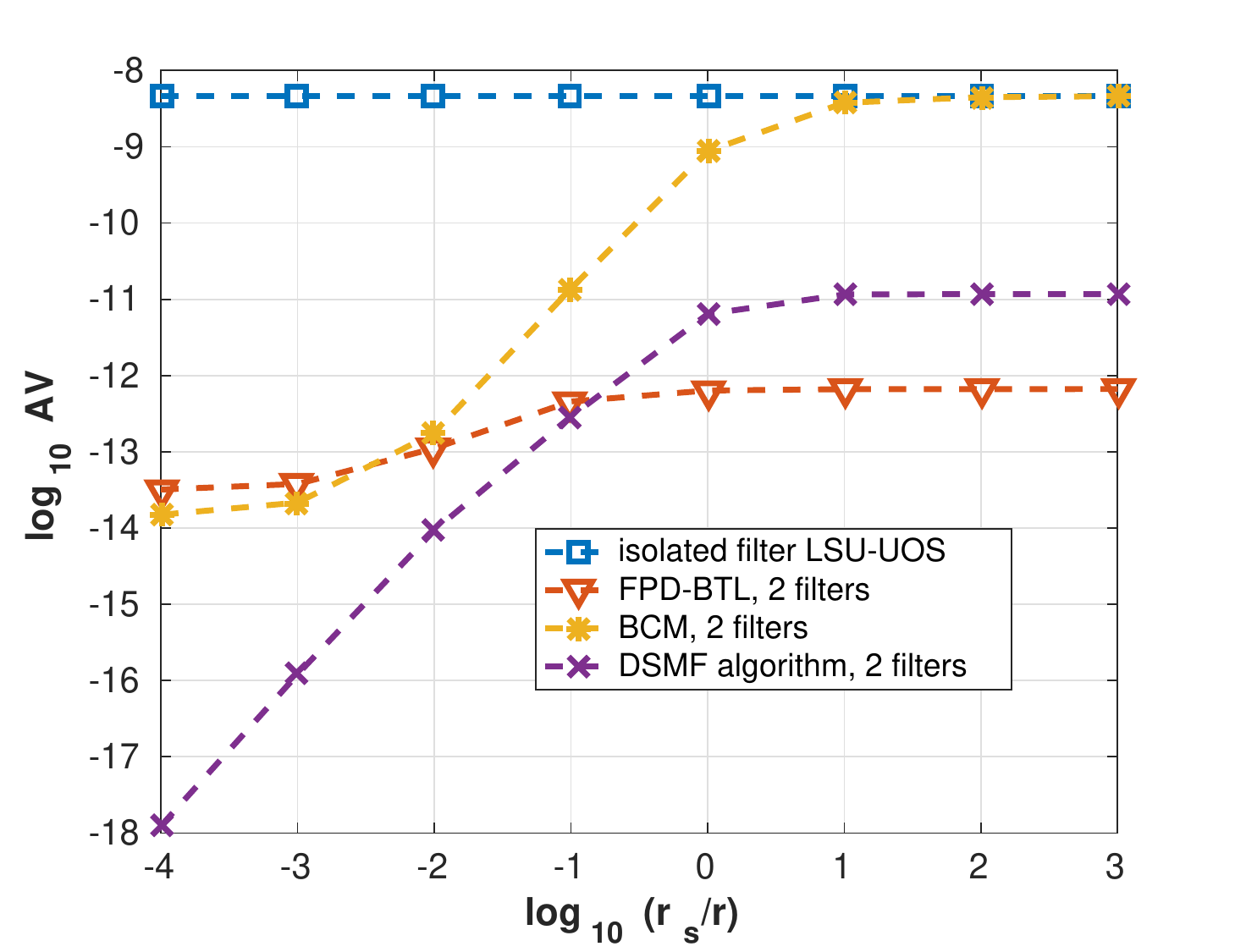}}
\hfill
{\includegraphics[width=0.49\textwidth]{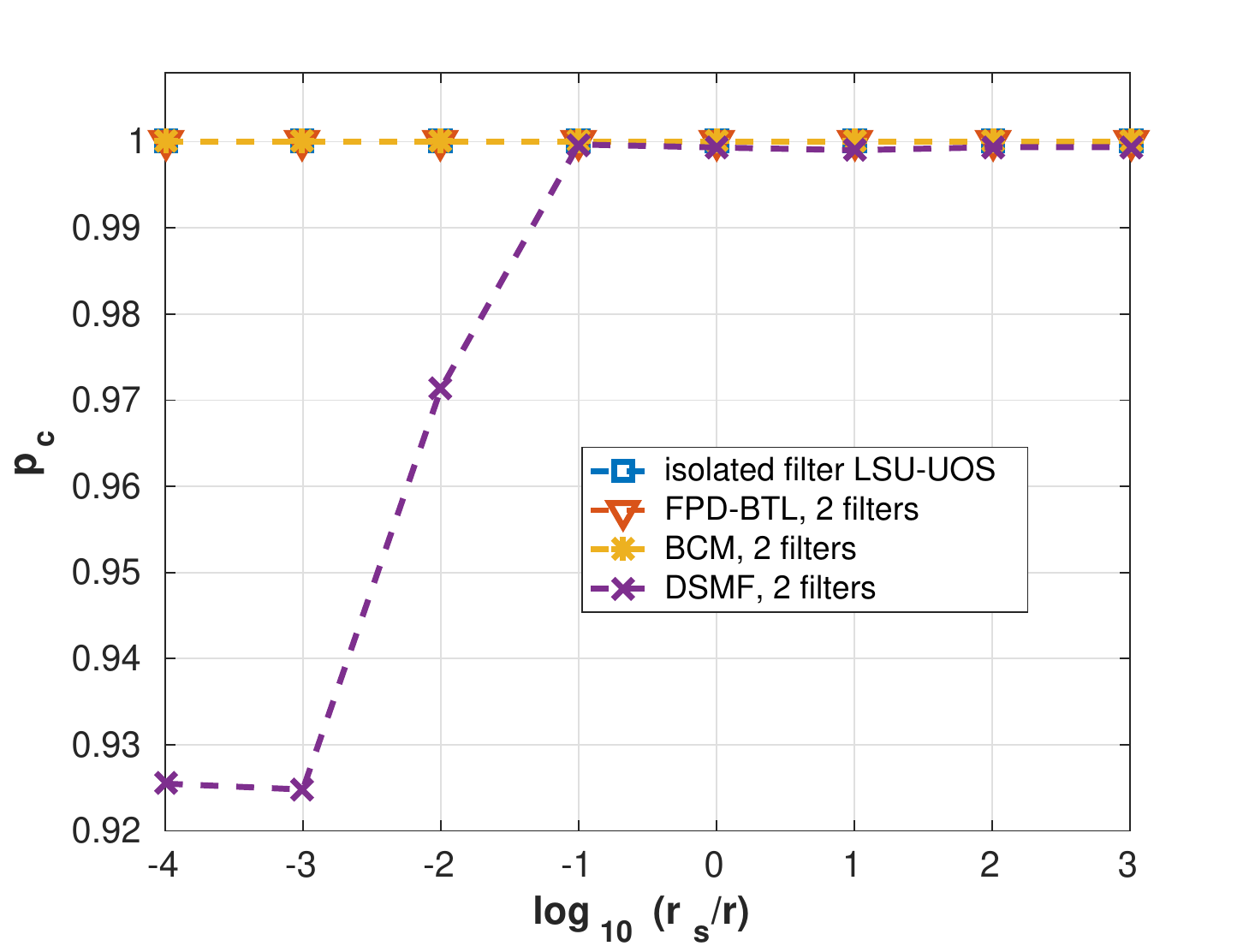}}
  \caption{Experiment \#4: a) TNSE~\eqref{e:deftnse}, b)
    AVR~\eqref{e:defvr}, c) AV~\eqref{e:defv} and d)
    $p\subs{c}$~\eqref{e:defcpr}, with dilation---$\sigma=1.4$---of
    the of eigenvalues of $A$~\eqref{e:another_system} in the target
    analysis model (FPD-BTL), and in the complete analysis model (BCM
    and DSMF. V-shaped complete model (Figure~\ref{f:shapegraphs}a)).
  }
  \label{f:mismatch_mult}
\end{figure}
Figure~\ref{f:mismatch_mult} displays all four performance measures
defined in Section~\ref{ss:evalucrit}. We note the following:
\begin{enumerate}[label=(\roman*)]
\item FPD-BTL exhibits strongly positive transfer below threshold
  (i.e.\ for $\ynpars_\sx>10\,\ynpars$ in this $n_\sx=1$ case), which
  is explained in the next paragraph.
\item For a wide range of $\ynpars_\sx/\ynpars$, of FPD-BTL performs
  better than the other methods.
\item Both methods based on complete modelling (i.e.\ BCM and DSMF)
  have a~similar TNSE (Figure~\ref{f:mismatch_mult}a).
\item DSMF exhibits negative transfer below threshold
  (Figure~\ref{f:mismatch_mult}b).
\item The containment probability, $p\subs{c}$, for DSMF rapidly falls
  (i.e.\ deteriorates) above threshold at $\ynpars_\sx=0.1\,\ynpars$
  (Figure~\ref{f:mismatch_mult}). The Bayesian methods maintain
  $p\subs{c}=1$, $\forall \ynpars_\sx/\ynpars$.
\end{enumerate}

There is an apparent saturation of the FPD-BTL performance below
threshold (Figures~\ref{f:mismatch_mult}a--c); i.e.\ those FPD-BTL
performance measures are far better (lower) than those of the isolated
filter, as $\ynpars_\sx/\ynpars$ increases, but saturating below
threshold.
This can be explained in the following way: if the source observation
noise, $\ynpars_\sx$, is high, its local data
update~\eqref{eqn:data-updt-unif2} via $y_{\sx,t}$ provides no
learning/concentration for $x_{\sx,t}$, because the source data
strip~\eqref{eqn:strip},~\eqref{eqn:data-updt-unif2} is a~superset of
the filtering inference~\eqref{eqn:data-updt-UOS-approx} prior state
predictor's orthotope~\eqref{eqn:time-updt-UOS-approx}.  Therefore,
the source evolves only via time updates. Effectively, its uncertainty
increases each step by the state noise half-width,
$\xnpars$~\eqref{e:defnpars}. For $\xnpars$ \emph{large}, the measure
of the source posterior support~\eqref{eqn:data-updt-UOS-approx} grows
quickly, and the transferred state predictor brings no more knowledge
to the target, after some filtering horizon, $t\subs{H}<\hi{t}$. In
his case, the target effectively becomes isolated from the source for
$t>t\subs{H}$.  For $\xnpars$ \emph{small}, the filtering
uncertainty~\eqref{eqn:data-updt-UOS-approx} grows more slowly, and,
effectively, $t\subs{H}>\hi{t}$. Had we increased $\hi{t}$ to
$t\subs{H}$ sufficiently, the FPD-BTL performances would, indeed, have
saturated at those if the isolated target, confirming that FPD-BTL is,
indeed, robust.

As $\sigma$ increases from 1 (analysis-synthesis matching) to 1.4, the
deviations of the performance measures from those at $\sigma=1$
increase (not illustrated).
For $\sigma<1$ and $\sigma>1.4$, these deviations are insensitive to
$\sigma$ as well, i.e.\ the effect of the eigenvalue dilation is most
significant in the interval $1<\sigma<1.4$.


Finally, note that these findings are influenced only minimally by
non-isotopic dilation of the eigenvalues of
$\asynt$~\eqref{e:another_system}, for instance the case in which only
the largest-radius eigenvalue is perturbed by $sigma$.

\paragraph{Experiment \#5: State noise mismatch}

Here, the U-shaped synthesis model (Figure~\ref{f:shapegraphs}b) is
adopted with system~\eqref{e:another_system}, modified so that $\asynt
\equiv 1.4$. The operating parameter, $\alpha>0$~\eqref{e:ugraph},
therefore controls the mismatch in the target's analytic model for
$x_t$ in FPD-BTL (Figure~\ref{f:shapegraphs}c) (i.e.\ the target
filter adopts $e_t=0$ in~\eqref{e:ugraph}), as well as in the complete
modelling (BCM and DSMF) which assume a~V-shaped analysis model (see
Figure~\ref{f:shapegraphs}a and
Section~\ref{sss:noisemis}). Specifically, in this experiment, $\alpha
= 0.4$, which was found to induce the maximal state mismatch for which
the FPD-BTL algorithm remained numerically stable.
\begin{figure}[h!]
 \centering
    {~\hfill a) \hspace{0.49\textwidth} b) \hfill ~} \\[-0.2em]
{\includegraphics[width=0.49\textwidth]{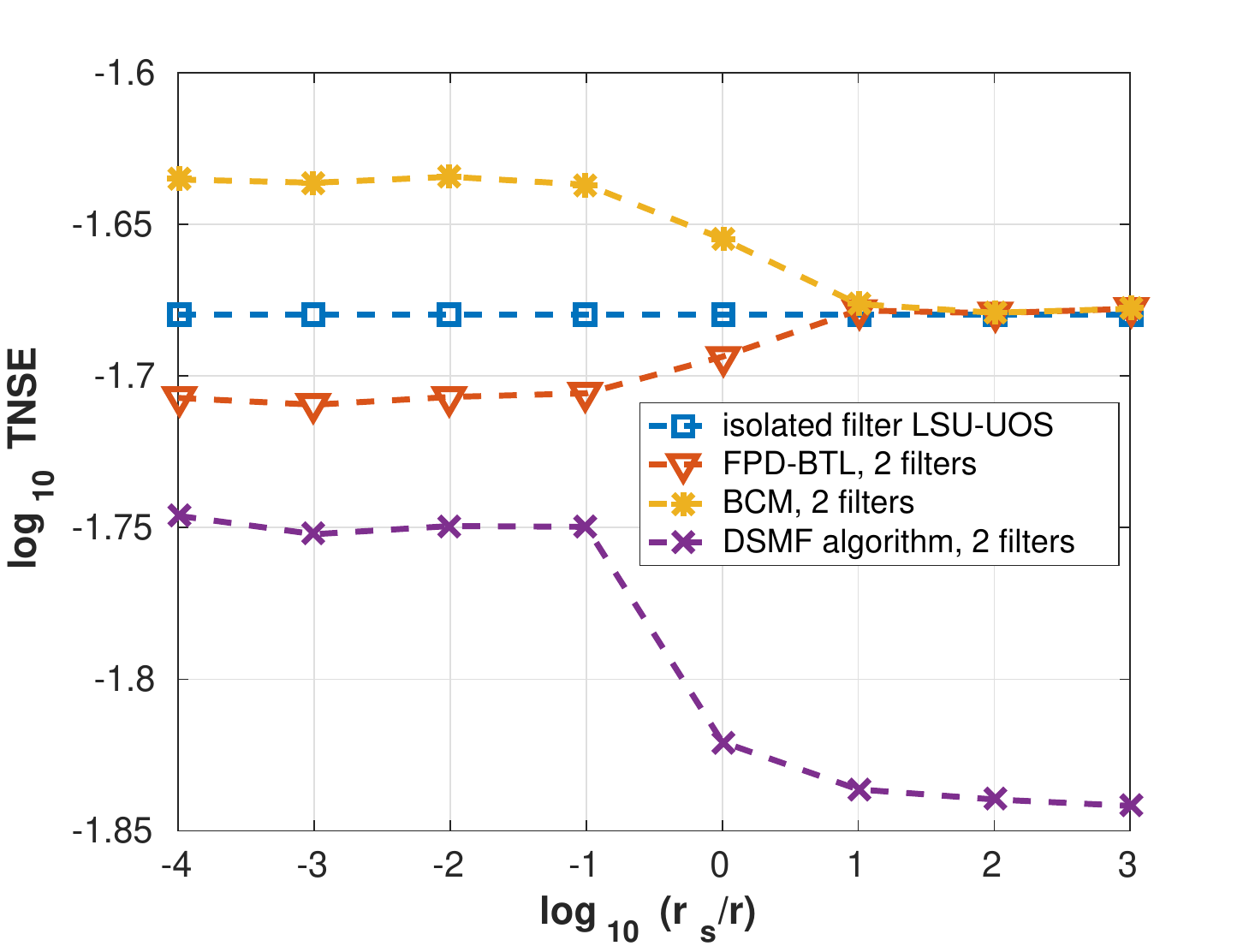}}
\hfill
{\includegraphics[width=0.49\textwidth]{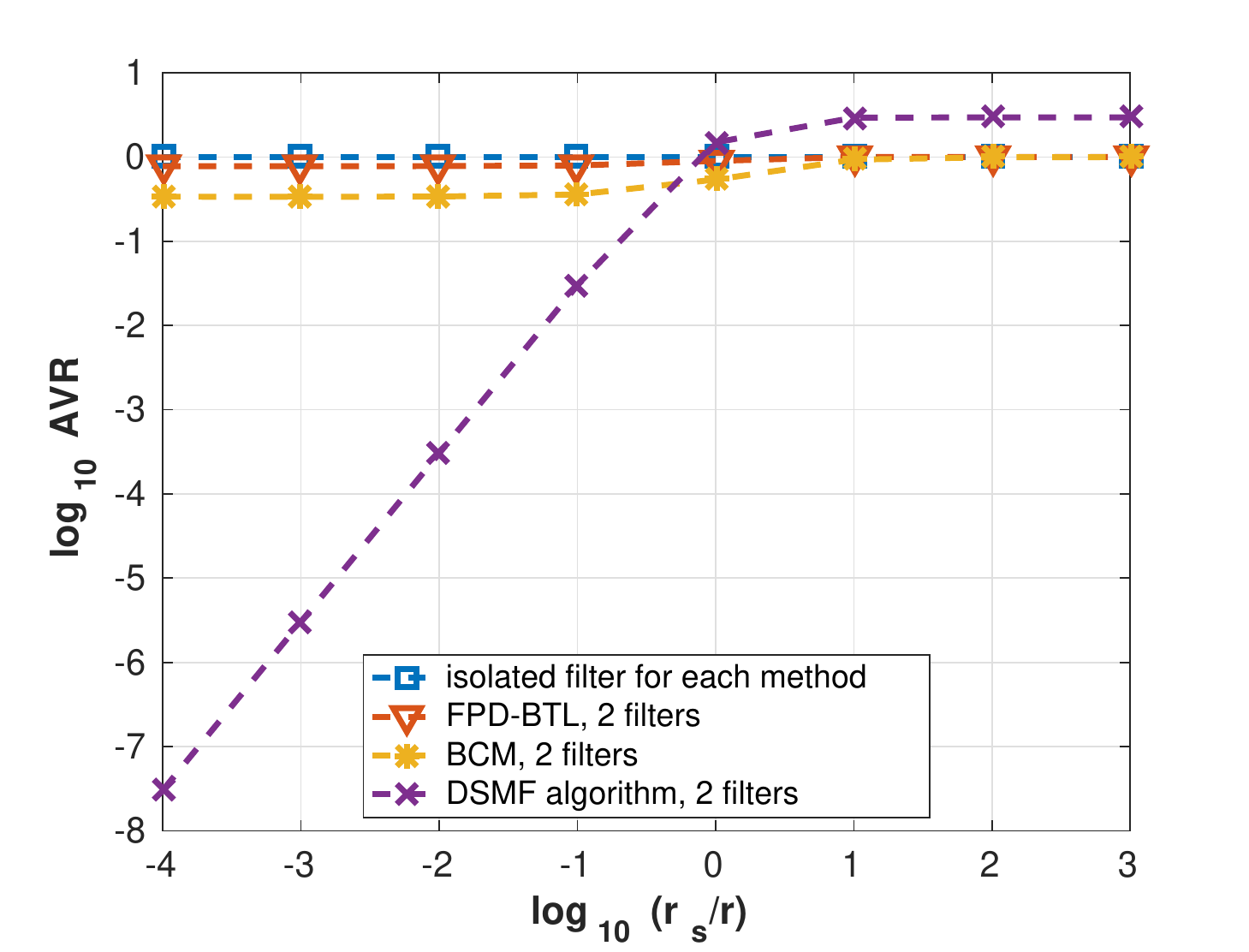}}\\
    {~\hfill c) \hspace{0.49\textwidth} d) \hfill ~} \\[-0.2em]
{\includegraphics[width=0.49\textwidth]{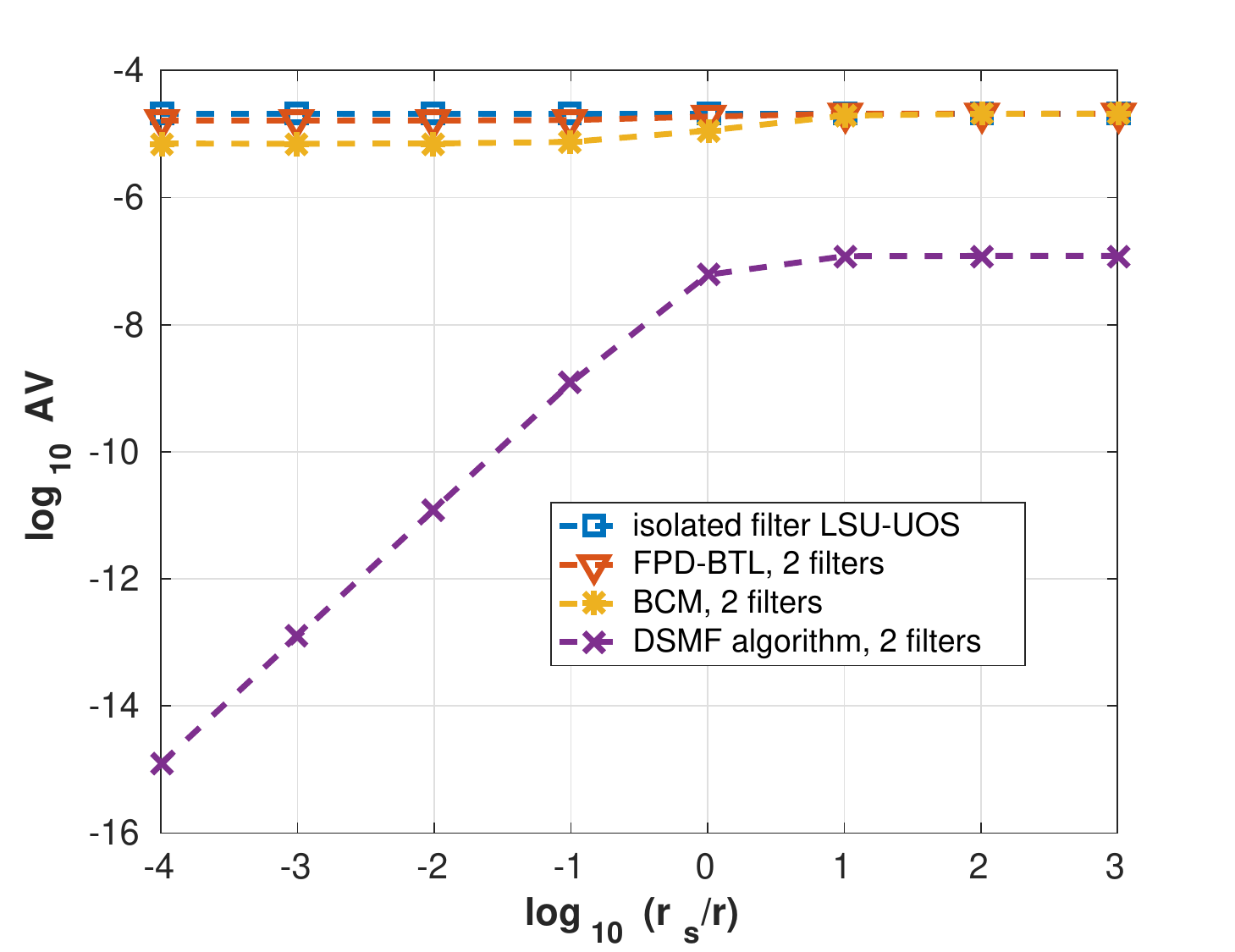}}
\hfill
{\includegraphics[width=0.49\textwidth]{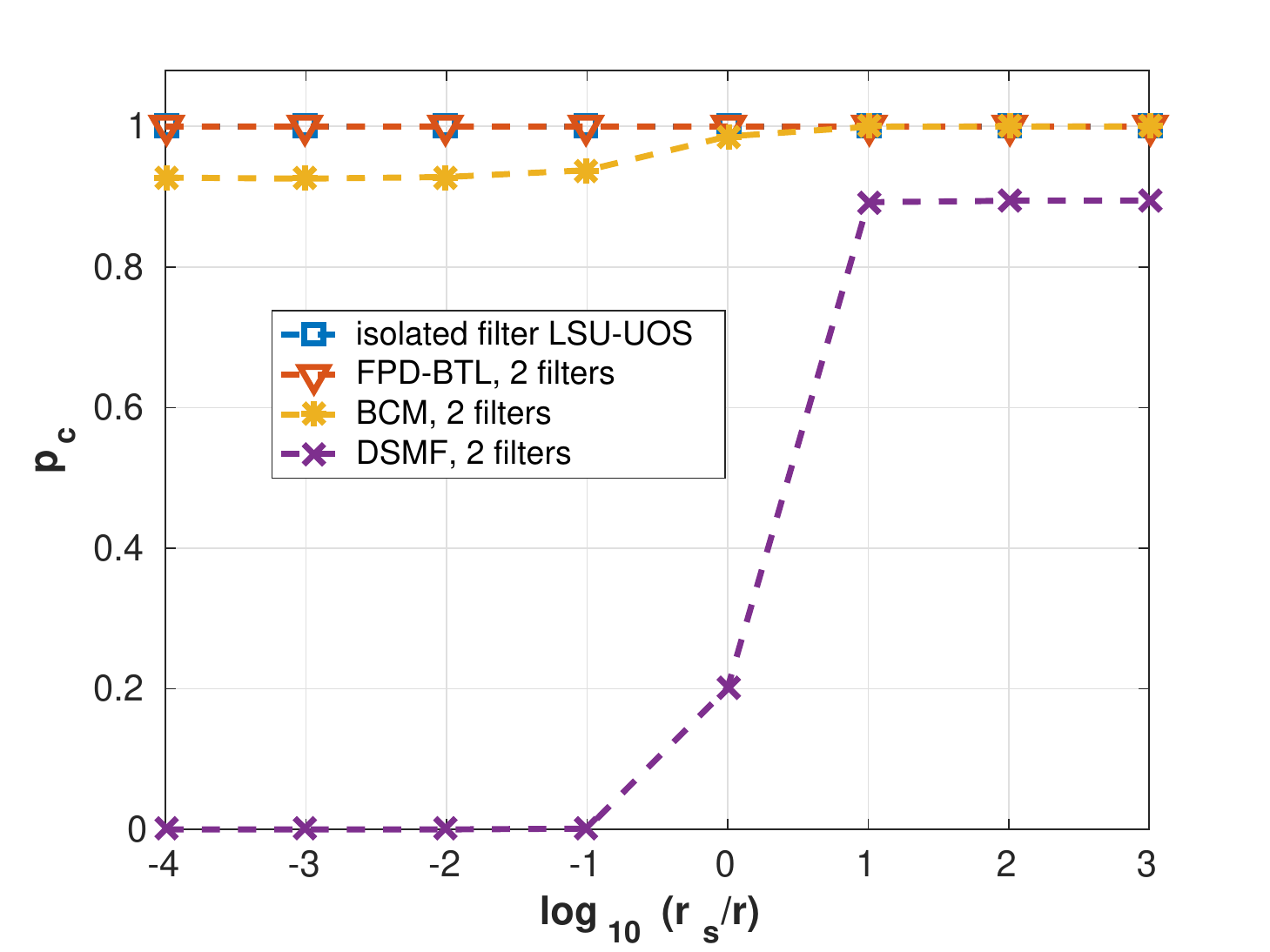}}\\
\caption{Experiment \#5: a) TNSE~\eqref{e:deftnse}, b)
  AVR~\eqref{e:defvr}, c) AV~\eqref{e:defv} and d)
  $p\subs{c}$~\eqref{e:defcpr} for a~U-shaped synthesis model
  (Figure~\ref{f:shapegraphs}b) with $\alpha=0.4$~\eqref{e:ugraph}
  and system~\eqref{e:another_system} with
  $\asynt=\aanal_\sx=\aanal=1.4\,A$.
 State analysis model mismatches in FPD-BTL target filter ($\alpha=0$)
 and in the complete analysis models of BCM and DSMF (V-shaped graphs
 with $\alpha=0$).
  %
  }
  \label{f:mismatch_noise}
\end{figure}
Also, $\ynpars=10^{-3}$, $\xnpars=10^{-2}$, $\lo{t}=50$, $\hi{t}=400$,
MC\,=\,200~runs.
All four performance measures (Section~\ref{ss:evalucrit}) are graphed
in Figures~\ref{f:mismatch_noise}a--d respectively. We note the
following:
\begin{enumerate}[label=(\roman*)]
\item The scale of the vertical axis in Figure~\ref{f:mismatch_noise}a
  is far smaller than in the TNSE graphs at earlier experiments;
  i.e.\ positive transfer is far smaller than in those earlier
  experiments.  Also, Figures~\ref{f:mismatch_noise}b
  and~\ref{f:mismatch_noise}c show similarly small positive transfer
  for the Bayesian methods (FPD-BTL and BCM).
\item In Figure~\ref{f:mismatch_noise}a, the complete modelling cases
  (BCM and DSMF) have opposing trends to those in earlier experiments.
  Furthermore, BCM exhibits negative transfer (for TNSE) \emph{above}
  threshold.
\item DSMF has negative transfer below threshold
  (Figure~\ref{f:mismatch_noise}b).
\item Although DSMF exhibits a~significant reduction in estimate
  uncertainty above threshold (Figure~\ref{f:mismatch_noise}c), its
  containment probability, $p\subs{c}$, falls to zero in this regime
  (Figure~\ref{f:mismatch_noise}d).  This underlines the fact that AVR
  and AV can be a misleading performance measures if quoted in the
  absence of $p\subs{c}$.
\end{enumerate}
This experiment demonstrates that all the tested alternatives to our
FPD-BTL algorithm perform unreliably when exposed to state noise
mismatch in the analysis model. In contrast, FPD-BTL demonstrates
robustness below the observation noise threshold,
$\ynpars_\sx>10\,\ynpars$, and weakly positive transfer above
threshold.

\subsection{Computational costs of the filtering algorithms}      \label{sss:times}
To assess the relative computational costs of the 3 algorithms,
$n\!=\!2$ scalar observation channels are synthesized via the V-shaped
graph (Figure~\ref{f:shapegraphs}a) with
system~\eqref{e:simple_system}.  Common state process,
$x_t$~\eqref{eqn:SS-model-state-evo}, is sequentially estimated via
FPD-BTL (Algorithm~\ref{alg:alg}), and by the fully-modelled
alternatives (BCM, Section~\ref{sss:bcm}, and DSMF,
Section~\ref{sss:dsmf}), now suppressing all the mismatches in
analysis which we studied in Section~\ref{ss:results}. The algorithms
are implemented in Matlab~2016a on an Intel i5-7500, 3.4\,GHz machine
hosting Linux operating system. All algorithms are run sequentially
(no parallel processing).  $\hi{t}=1\,000$ and MC\,=\,1\,000~runs per
setting of $\ynpars_\sx$ (there were 8~such settings). The runtimes
were 946\,s (FPD-BTL), 625\,s (BCM) and 1\,605\,s (DSMF).  The
comparative run-time for the isolated LSU-UOS filtering
algorithm~\eqref{eqn:data-updt-UOS-approx},
\eqref{eqn:time-updt-UOS-approx} processing only scalar target data
channel, $d_t$ (Figure~\ref{f:shapegraphs}c), was 598\,s.

With $n\geq 2$ observation channels, BCM and DSMF compute just one
time update per $n$ data updates per step of these algorithms. A
fusion operation is also required in every step of DSMF. Meanwhile, in
FPD-BTL, the source and target tasks each involve one data- and
time-update step (lines II. and III. in Algorithm~\ref{alg:alg},
respectively), with the single transfer being computationally trivial
(set intersection~\eqref{e:xconstr}).  Note also that FPD-BTL requires
only that a UOS state predictor,
$f_\sx(x_{\sx,t}|d_\sx(t-1))$~\eqref{eqn:time-updt-UOS-approx} be made
available to the target in each step in the algorithm (line~I. of
Algorithm~\ref{alg:alg}), but does not stipulate how the source
actually calculates this. In this sense, the source data- and
time-updates can be omitted from the target's computational budget,
approximately halving the FPD-BTL runtime quoted above, and
significantly outperforming the completely modelled alternatives (BCM
and DSMF), which are centralized algorithms. This parallel nature of
the FPD-BTL algorithm---parallelizable filtering operations with
a~trivial fusion/transfer operation (i.e. intersection) at the central
target task---is a key attribute, recommending FPD-BTL in
resource-critical applications that arise in sensorized
environments~\cite{CheHoYu:17}.

\subsection{Discussion}                                   \label{ss:discussion}

The key distinction between our optimal Bayesian transfer learning
algorithm (FPD-BTL), and the fully-modelled Bayesian (BCM) and fusion
(DSMF) alternatives, lies in the fact that the source LSU-UOS task is
an independent modeller, transferring only its local state predictor,
$f_\sx(x_{\sx,t}|d_\sx(t-1))$~\eqref{e:fac1}---i.e. its local
(approximate) Bayesian sufficient statistics, $\lo{x}^+_t$
and~$\hi{x}^+_t$
\eqref{eqn:time-updt-UOS-approx}---for processing by the target
task. Our aim with these experiments has been to assess whether this
provides robustness to mismatches that can affect the fully modelled
alternatives, since FPD-BTL allows expert local (and independent)
source knowledge to be transferred in support of the target.

We empirically defined a~\emph{threshold} using the source observation
noise parameter, $\ynpars_\sx$, and the target observation noise
parameter, $\ynpars$~\eqref{e:defnpars}, at $\ynpars_\sx=10\,\ynpars$
for $n_\sx=1$ source (see, for example,
Figure~\ref{f:volr_simple_system_compl}). FPD-BTL has \emph{positive
  transfer} above threshold ($\ynpars_\sx\leq 10\,\ynpars$) and it is
\emph{robust} (i.e. the transfer is rejected and the target reverts to
an isolated LSU-UOS task) below threshold ($\ynpars_\sx\geq
10\,\ynpars$). In the multi-source case ($n_\sx>1$), the threshold
depends on $n_\sx$ as $\ynpars_\sx\!=\!10n_\sx\, \ynpars$ (see
Figure~\ref{f:tnse_vol_simple_system}).

In the case of two filtering tasks ($n=n_\sx+1=2$), with matching of
the synthetic and analytic models in both filters, FPD-BTL is robust
below threshold and shows positive transfer above threshold, where its
performance measures---TNSE~\eqref{e:deftnse}, AVR~\eqref{e:defv} and
AV~\eqref{e:defvr}---saturate
(Figure~\ref{f:volr_simple_system_compl}). The complete modelling
methods (BCM and DSMF) saturate only in TNSE and not in uncertainty
measures (AVR and AV), where they depend on $r_\sx$ linearly, since
they process raw source observations, $y_{\sx,t}$, using the source's
known observation model (both of which do not need to be processed in
FPD-BTL).  Also, if the synthetic and analytic models match, the DSMF
algorithm yields improved TNSE below threshold, compared to BCM, due
to tighter geometric approximation of the state set, $\mS{x}_t$
(Section~\ref{ss:evalucrit}), by a~general ellipsoid than by
a~(conservative) orthotope.

The DSMF algorithm needs a burn-in period because it is more sensitive
to its initialization, unlike the Bayesian LSU-UOS methods. Therefore,
the lower bound for the performance evaluation epoch had to be set as
$\lo{t}\gg1$ for DSMF in section~\ref{ss:evalucrit}.  Practically,
$\lo{t}=\hi{t}/2$ was used. Without this burn-in, the performance
criteria for DSMF were all much worse, and, in fact, suffered
a~negative transfer.

In Section~\ref{sss:modmismatch}, we studied synthesis-analysis
matching in the source ($n_\sx=1$) and mismatching in the target.
FPD-BTL maintained its positive transfer above threshold and
robustness below threshold in all the studied cases, including sure
containment probability, $p\subs{c}=1$~\eqref{e:defcpr}. Deterioration
of filtering performance for the complete modelling methods (BCM and
DSMF)---and their non-robustness to modelling mismatch---occurred in
different ways, depending on the particular system and type of
mismatch. For instance, note the greatly increased positive transfer
in FPD-BTL vs BCM in Figure~\ref{f:mismatch_rot1}; in FPD-BTL vs both
BCM and DSMF in Figures~\ref{f:mismatch_mult}
and~\ref{f:mismatch_noise}; note the negative transfer in
Figures~\ref{f:mismatch_mult}b (DSMF), \ref{f:mismatch_noise}a (BCM),
and~\ref{f:mismatch_noise}b (DSMF), compared to positive and robust
transfer performance in FPD-BTL in these experiments; and, finally,
note this $p\subs{c}\rightarrow 0$ for DSMF in
Figure~\ref{f:mismatch_noise}d, while it remains at unity for FPD-BTL.

The state support set intersections, \eqref{eqn:data-updt-unif2}
and~\eqref{e:xconstr}, were never empty if the target synthesis and
analysis models matched.  However, in the cases of target model
mismatch (Section~\ref{sss:modmismatch}), the values of the mismatch
operating parameters ($\varphi$ in Experiment~\#3, $\sigma$ in
Experiment~\#4, and~$\alpha$ in Experiment~\#5) were chosen
appropriately, so that the target synthesis and analysis models did
not differ excessively. Otherwise, the intersections,
\eqref{eqn:data-updt-unif2} and~\eqref{e:xconstr}, may have been
empty, and transfer stopped (Theorem~\ref{t:t1a}, final statement). It
was observed that intersections in the data
update~\eqref{eqn:data-updt-unif2} had a tendency to be empty more
often (and it was more sensitive to inappropriate values of the
operating parameters) than in the transfer~\eqref{e:xconstr}. In
experiments, these cases checked and discarded.

In summary, the FPD-BTL algorithm has the following properties:
\begin{itemize}
  \item it does not require a complete model of the source, nor its
    interaction with the target;
  \item in analysis, the source and target are independent tasks
    (multiple modelling),
  \item it requires transfer only of the UOS source state
    predictor~\eqref{eqn:time-updt-UOS-approx} as a local expert
    Bayesian sufficient statistic for processing by the target; these
    statistics---$\lo{x}^+_t$ and $\lo{x}^-_t$---summarize all that is
    relevant in $d_\sx(t)$, $A_\sx$, $B_\sx$, $C_\sx$, for transfer
    learning in the target;
  \item in the case of target analysis model mismatch, the target
    accepts the source state predictor, avoiding imposition of this
    mismatch on the source, as occurs in the complete modelling
    strategies (BCM and DSMF);
  \item it readily scales to transfer of state predictors from
    multiple ($n_\sx>1$) sources;
  \item its threshold occurs at $\ynpars_\sx\!=\!10\,n_\sx\, \ynpars$,
    where $n_\sx$ is number of sources; therefore, increasing the
    number of sources can compensate for poorer quality of individual
    source filters;
  \item the containment probability, $p\subs{c}$, is unity and never
    collapses, even when the target analysis and synthesis models
    mismatch;
  \item the computational costs of FPD-BTL at the target are far
    smaller than the costs of the complete modelling methods at the
    processing/fusion centre. Local data are processed locally
    (intrinsic parallelization), and the FPD-optimal BTL operator for
    LSU-UOS class---i.e.\ the required minimization of
    Kullback-Leibler divergence---reduces to intersection of
    supporting orthotopes at the target.
\end{itemize}
\section{Conclusion}                                   \label{sec:conclusion}

In this paper, we have proposed a Bayesian transfer learning (BTL) algorithm based on
fully probabilistic design (FPD), and have applied it to transfer between filtering tasks described by linear state-space models with uniformly distributed noises on orthotopic support (i.e. LSU-UOS filters).

As already noted, FPD-BTL is truly a \emph{transfer} learning technique,
as opposed to a joint/multi-channel inference scheme, since the source
is an independent modeller from the target (multiple modelling) and it transfers the
result of its local Bayesian inference task---state prediction---to support the inference task (i.e. state filtering) at the target,
in line with the core notions of transfer learning \cite{PanYan:10}.
FPD-BTL is also a truly \emph{Bayesian} transfer learning algorithm,
where the source's fully probabilistic inference, $f_\sx(x_{\sx,t}|d_\sx(t-1))$
(Figure \ref{fig:transfer-diagram}), is transferred, rather than the source data,
$d_\sx(t)$, or merely their statistics.
Therefore, a hallmark of BTL is that the target processes the source's (Bayesian)
sufficient statistics and  local modelling knowledge,  $f_\sx(\cdot)$,
rather than requiring the transfer of the raw source data. We can see this,
for instance, in the second line of transfer step I. in Algorithm \ref{alg:alg},
which processes UOS state predictor \eqref{eqn:time-updt-UOS-approx}, with shaping parameters (statistics), $\lo{x}^{+}_{t+1}$ and $\hi{x}^{+}_{t+1}$, rather then the joint processing
of $d_\sx(t)$ and $d(t)$ in  BCM  \eqref{e:dupdc}.

FPD-BTL generally involves mean-field-type integrals which can be hard to compute.  Its specialisation in LSU-UOS filtering tasks results in the FPD-justified geometric intersection operator for the transfer step, which is computationally trivial for orthotopes.  Moreover, this operator
can be readily extended to multiple source tasks (Theorem \ref{t:t2}), with exceptional implementational speed (i.e. minimal increase in computational cost with $n_{\sx}>1$ ).

The proposed BTL operator is a function of the source uncertainty
(measured as variance or entropy), and not simply the source's state estimate. It is this which enables robust transfer, i.e. the rejection of poor-quality source knowledge.
In addition, the application of FPD-BTL to the UOS class has proved to be
theoretically significant, since it demonstrates that the
non-robustness of earlier Gaussian-based FPD-optimal BTL schemes,
e.g. in \cite{FolQui:18}, is a consequence of higher-moment loss in the
induced mean-field operator, and is not an intrinsic limitation of FPD
methodology.

Finally, note that the the proposed BTL technique can be used to transfer from any source learning tasks (filters) that yield sequential state predictors, including particle filters, point-mass filters, etc.

The results presented here can be further developed in the following ways:
\begin{itemize}
\item Currently, there is no measure of trust in the reliability of
the source(s): the target fully accepts the source knowledge
\eqref{e:fac1}. The future research will focus on the target's (Bayesian) modelling of the sources' reliability.
\item   
  Further research is required into the case where the sets in the transfer~\eqref{e:xconstr} and data updates~\eqref{eqn:data-updt-unif2}  are disjoint. This can probably be resolved  via a~sub-unity transfer weight. 
\item The isolated LSU-UOS filtering algorithm (Section \ref{sec:prelim}) comprises two local
  approximations per filtering step. The accumulated approximation error can be reduced via more flexible geometric support approximations, such as zonotopes \cite{TanWanZhaShe:20},
  \cite{ScoRaiMarBra:16}, with bounding guarantees available via global approximations based on empirical approximations (e.g. particle filters~\cite{Lietal:16}).
\item In future work, the target can itself be a joint network modeller, independent of the source(s),  as in  \cite{PapQui:21}, further enhancing the positive transfer.
\item The currently adopted FPD-BTL scheme involves static transfer, in the sense that the
  source knowledge is transferred in the form of the marginal state predictor at each time step. By transferring joint distributions over multiple time steps---i.e. dynamic transfer---the source's temporal (dynamic) knowledge can be exploited at the target \cite{PapQui:18}.
\end{itemize}

\section*{Acknowledgment}
This research has been supported by The Czech Science Foundation (GA\v{C}R) [grant 18-15970S].

\appendix

\section{Proofs}                           \label{s:app}

\subsection{Proof of Theorem~\ref{t:t1a}}              \label{a:proof1}

Let us partition the integration domain, $\mS{x}_{\sx,t}$, of
$f_\sx(x_{t}|d_\sx(t-1))$ into sets $\mS{x}^\cap_{t}$ and
$\mS{x}^\bullet_{t}$ (see Figure~\ref{f:venn12}), so that
$\mS{x}^\cap_{t}=\mS{x}_{\sx,t}\cap\mS{x}_{t}$ and
$\mS{x}^\bullet_{t}=\mS{x}_{\sx,t}\cap\mS{x}_{t}^c$, where
$\mS{x}_{t}^c$ is the complement of $\mS{x}_{t}$.
Alternatively:
$\mS{x}^\cap_{t}=\left\{x_t\left|\,x_t\in\mS{x}_{\sx,t} \land x_t\in\mS{x}_{t} \right.\right\}$
and
$\mS{x}^\bullet_{t}=\left\{x_t\left|\,x_t\in\mS{x}_{\sx,t} \land x_t\not\in\mS{x}_{t} \right.\right\}$.
It holds that $\mS{x}^\cap_{t} \cap \mS{x}^\bullet_{t} = \emptyset$ and
$\mS{x}^\cap_{t} \cup \mS{x}^\bullet_{t} = \mS{x}_{\sx,t}$.
The set, $\mS{x}^\bullet_{t}$, is
open on the border with $\mS{x}_t$, i.e.
$\mS{x}^\bullet_{t}\cap\mS{x}_t=\emptyset$.

Also, $\mS{y}_{t}|x_t$ is the support of $f(y_t|x_t)$ and
$\breve{\mS{y}}_{t}|x_t$ of $\breve{f}(y_t|x_t,f_\sx)$, both
a function of $x_t$.
The set $\mS{y}_{t}|x_t\neq\emptyset$,
$\forall x_t\in\mathbb{R}^\xsize$,~\eqref{e:LSU-dupdt}.

The KLD~\eqref{e:klddef} can be written, using Fubini's theorem, as
  \begin{equation}                         \label{e:klddecomp}
  \kld{\breve{f}}{f^I}  =
\! \int\limits_{\mS{x}_{\sx, t}} \!\! f_s(x_{t}|d_\sx(t-1))
\,\kldsym_y(x_t) \diff x_t + \kldsym_{x1} + \kldsym_{x2},
  \end{equation}
where
\begin{eqnarray*}
  \kldsym_y(x_t) &=&
  \!\int\limits_{\breve{\mS{y}}_{t}|x_t}\!\! \breve{f}(y_{t}|x_t,f_\sx)\,
  \ln\frac{\breve{f}(y_{t}|x_t,f_\sx)}{f(y_{t}|x_{t})} \diff y_t,  \nonumber \\
  \kldsym_{x1} &=& \!\int\limits_{\mS{x}^\cap_{t}} \!\!
  f_\sx(x_{t}|d_\sx(t-1))
  \,\ln\frac{f_\sx(x_{t}|d_\sx(t-1))}{f(x_{t}|d(t-1))}
  \diff x_t, \nonumber \\
  \kldsym_{x2} &=& \!\int\limits_{\mS{x}^\bullet_{t}} \!\!
  f_s(x_{t}|d_\sx(t-1))
  \,\ln\frac{f_\sx(x_{t}|d_\sx(t-1))}{f(x_{t}|d(t-1))}
  \diff x_t. \nonumber \\
\end{eqnarray*}

Let us consider three (exhaustive) cases of the
relationship between the sets $\mS{x}_{\sx,t}$
and~$\mS{x}_{t}$:
\begin{enumerate}[(i)]
\item \label{itemi} $\mS{x}_{\sx,t} \subseteq \mS{x}_{t}$,
  i.e. $\mS{x}^\bullet_{t}=\emptyset$, which is a~special case
  (Figure\,\ref{f:venn12}, case\,\ref{itemi}). Then
  $\kldsym_{x2}=0.$ Also, $\kldsym_{x1}$ is a~non-negative, finite constant
  independent of~$\breve{f}(y_t|x_t,f_\sx)$. The only term influencing
  $\kld{\breve{f}}{f^I}$ is $\kldsym_y(x_t)$. To be finite, it must
  hold that $\breve{\mS{y}}_{t}|x_t \subseteq \mS{y}_{t}|x_t$, where
  $\mS{y}_{t}|x_t$ is the support of~$f(y_{t}|x_{t})$. To minimize
  $\kldsym_y(x_t)$, we must minimize the normalization constant of
  $\breve{f}(y_t|x_t)$, i.e.\ extend maximally the support of
  $\breve{f}(y_t|x_t,f_\sx)$ to maximize its measure relating
  $\breve{\mS{y}}_{t}|x_t = \mS{y}_{t}|x_t$. Then,
  ${f}^o(y_t|x_t,f_\sx) \propto f(y_t|x_t)$, given
  $x_t\in\mS{x}^\cap_{t}\ (\equiv\mS{x}_{\sx,t})$. After substitution
  into~\eqref{e:optimized}, the FPD-optimal joint pdf ${f}^o\propto
  f(y_t|x_t)\, f_\sx(x_{t}|d_\sx(t-1))$, i.e. it is zero $\forall
  x_t\not\in\mS{x}^\cap_{t}$.
\item \label{itemii} $(\mS{x}_{\sx,t} \not\subseteq \mS{x}_{t}) \land
  (\mS{x}_{\sx,t}\cap\mS{x}_{t}\neq\emptyset)$, i.e.\ both
  $\mS{x}^\cap_{t}$ and~$\mS{x}^\bullet_{t}$ are non-empty
  (Figure\,\ref{f:venn12}, case\,\ref{itemii}). In this general
  case, we adopt ${f}^o(y_t|x_t,f_\sx)\propto f(y_t|x_t)$,
  $\forall x_t\in\mS{x}^\cap_{t}$, in consequence of~\ref{itemi}, and prove
  its optimality here. We introduce the following sequences,
  $k\in\mathbb{N}$:
  \begin{itemize}
  \item $f^{(k)}(x_t|d(t-1))>0,\ \forall x_t\in\mS{r}^\xsize$ and $\forall k$,
  with $\lim\limits_{k\rightarrow+\infty} f^{(k)}(x_t|d(t-1)) =
  f(x_t|d(t-1))$,
  \item $\breve{f}^{(k)}(y_{t}|x_t,f_\sx)\propto
  f(y_t|x_t)\,\xi^{(k)}(x_t)$, where $\xi^{(k)}(x_t)>0,\ \forall
  x_t\in\mS{r}^\xsize$ and $\forall k$, with
  $\lim\limits_{k\rightarrow+\infty} \xi^{(k)}(x_t)
  = \chi(x_t \in \mS{x}^\cap_{t})$.
  \end{itemize}
These sequences converge to uniform pdfs.

The sequences defined above are chosen so that
$0<\frac{\breve{f}^{(k)}(y_{t}|x_t,f_\sx)}{f^{(k)}(x_t|d(t-1))}\leq
a^{(k)}<+\infty$ $\forall x_t\in\mS{x}_{s,t}$,
$y_t\in \breve{\mS{y}}_{t}|x_t$ and
$\lim\limits_{k\rightarrow+\infty}a^{(k)}<+\infty$.  This
choice guarantees a~finite
KLD~\eqref{e:klddef}.

The definitions above imply
\begin{eqnarray}
  \kldsym_y^{(k)}(x_t) &=&
  \!\int\limits_{\breve{\mS{y}}_{t}|x_t}\!\! \breve{f}^{(k)}(y_{t}|x_t,f_\sx)\,
  \ln\frac{\breve{f}^{(k)}(y_{t}|x_t,f_\sx)}{f(y_{t}|x_{t})} \diff y_t,  \nonumber \\
  \kldsym^{(k)}_{x1} &=& \!\int\limits_{\mS{x}^\cap_{t}} \!\!
  f_\sx(x_{t}|d_\sx(t-1))
  \,\ln\frac{f_\sx(x_{t}|d_\sx(t-1))}{f^{(k)}(x_{t}|d(t-1))}
  \diff x_t. \nonumber
\end{eqnarray}
We define the sequence, $\kld{\breve{f}}{f^I}^{(k)}$, where
$\lim\limits_{k\rightarrow+\infty}\kld{\breve{f}}{f^I}^{(k)}=\kld{\breve{f}}{f^I}$,
and, similarly to~\eqref{e:klddecomp}, we can express it formally as
\begin{eqnarray}                                           \label{e:kldseq}
\hspace{-2em}
\kld{\breve{f}}{f^I}^{(k)}\!\!\!\!\!
  & = &
  \underbrace{
    \!\!\! \int\limits_{\mS{x}^\cap_{t}} \!\! f_\sx(x_{t}|d_\sx(t-1))\,\kldsym_y^{(k)}(x_t) \diff x_t
    + \kldsym^{(k)}_{x1}
  }\limits_{S^{(k)}} + \nonumber\\
&+&
  \underbrace{
    \!\!\!\int\limits_{\mS{x}^\bullet_{t}}\!\!f_\sx(x_{t}|d_\sx(t-1))
   \!\!\!\int\limits_{\breve{\mS{y}}_{t}|x_t}\!\!\!\!\breve{f}^{(k)}(y_{t}|x_t,f_\sx)\,
   \ln\frac{f_\sx(x_t|d_s(t-1))}{f(y_t|x_t)}\diff y_t \diff x_t
  }\limits_{Q^{(k)}}+\nonumber \\
&+&
  \underbrace{
    \!\!\!\int\limits_{\mS{x}^\bullet_{t}}\!\!f_\sx(x_{t}|d_\sx(t-1))
   \!\!\!\int\limits_{\breve{\mS{y}}_{t}|x_t}\!\!\!\!\breve{f}^{(k)}(y_{t}|x_t,f_\sx)\,
   \ln\frac{\breve{f}^{(k)}(y_{t}|x_t,f_\sx)}{f^{(k)}(x_t|d(t-1))}
   \diff y_t
   \diff x_t.
  }\limits_{R^{(k)}}
  \nonumber\\
   ~
\end{eqnarray}
The term $S^{(k)}$ in~\eqref{e:kldseq} can be taken directly to the
limit ($k\rightarrow+\infty$) and treated as in case~\ref{itemi},
with the optimal result
${f}^o \propto f(y_t|x_t)\, f_\sx(x_{t}|d_\sx(t-1))$,
$\forall x_t\in\mS{x}^\cap_{t}$.
According to the definitions of the sequences, $\breve{f}^{(k)}(y_{t}|x_t,f_\sx)$
and~$f^{(k)}(x_t|d(t-1))$, and recalling that $\mS{x}^\bullet_{t}$ is
open, the terms $Q^{(k)}$ and~$R^{(k)}$ in~\eqref{e:kldseq} converge
pointwise to zero for $x_t\in\mS{x}^\bullet_{t}$.
%
Therefore, if $k\rightarrow+\infty$,
KLD~\eqref{e:kldseq} is finite $\forall x_t\in\mS{x}^\bullet_{t}$.
Hence, given $x_t\in\mS{x}^\bullet_{t}$,
%
%
${f}^o(y_t|x_t,f_\sx)
=\lim\limits_{k\rightarrow+\infty}\breve{f}^{(k)}(y_{t}|x_t,f_\sx)=0$
for all $y_t\in\mS{r}^\ysize$, i.e.\
$\breve{f}=0$ for
$x_t\in\mS{x}^\bullet_{t}$.
If $x_t\not\in\mS{x}_{\sx,t}$, then
$\breve{f}=0$
also~\eqref{e:optimized}.
Without changing
$\breve{f}$,
we can define ${f}^o(y_t|x_t,f_\sx)
\propto f(y_t|x_t)\,\chi(x_t \in \mS{x}^\cap_{t})$.
\item                               \label{itemiii}
$\mS{x}_{\sx,t}\cap\mS{x}_{t}=\emptyset$, then
 $\mS{x}^\cap_{t}=\emptyset$ and ${f}^o(y_t|x_t,f_\sx) \equiv 0$,
 $\forall x_t\in\mS{r}^\xsize$ (Figure\,\ref{f:venn12},
  case\,\ref{itemiii}).
In this case,
as a consequence of the target's full trust in the source's state predictor
(see the first bullet point in Section~\ref{sec:conclusion}),
transfer learning is impossible and we define
${f}^o(y_t|x_t,f_\sx) \equiv f(y_t|x_t)$,
i.e. ${f}^o(y_t|x_t,f_\sx)$ is conditionally independent of $f_\sx$,
and $\mS{x}^o_{t}\equiv\mS{x}_{t}$.
\end{enumerate}

To conclude the proof, \emph{if} $\mS{x}^\cap_t=\emptyset$, then
the sought pdf,
${f}^o(y_t|x_t,f_\sx)$,
conditioned on the transferred source state predictor, $f_\sx$,
is proportional to $f(y_t|x_t)$,
if $x_t\in\mS{x}^\cap_t$, and it is zero if $x_t\not\in\mS{x}^\cap_t$;
i.e.\ ${f}^o(y_t|x_t,f_\sx)\propto
f(y_t|x_t\in(\mS{x}_{\sx,t}\cap\mS{x}_{t}))$. Hence,
$\mS{x}^o_{t}=\mS{x}^\cap_{t}$.
Conversely, \emph{if} $\mS{x}^\cap_{t}=\emptyset$, then ${f}^o(y_t|x_t,f_\sx) \equiv
f(y_t|x_t)$ (i.e.\ the isolated target observation model)
and $\mS{x}^o_{t}\equiv\mS{x}_{t}$.
%


\subsection{Proof of Theorem~\ref{t:t2}}           \label{a:proof2}


Let $n=3$. The filter 1 is the target filter and filters~2 and~3 are the
source filters. Perform transfer learning between the filters~2 and~3
by intersecting their state supports (Theorem~\ref{t:t1a}). We obtain
an abstract filter providing all the
knowledge in filters~2 and~3 available for transfer to filter~1.
The transfer between the
abstract filter and filter~1, then results in another intersection. As
intersection is a commutative and associative operation, we
obtain~\eqref{e:xconstrnet}. The same procedure can be applied for any
$n\geq3$.


\bibliographystyle{ieeetr} 
\bibliography{ref}

\end{document}